\PassOptionsToPackage{sort}{natbib}
\documentclass{ecai} 

\usepackage{latexsym}

\usepackage{times}
\usepackage{soul}
\usepackage{url}
\usepackage[utf8]{inputenc}
\usepackage[small]{caption}
\usepackage{subcaption}
\usepackage{graphicx}
\usepackage{amsmath}
\usepackage{amsthm}
\usepackage{booktabs}
\usepackage{algorithm}
\usepackage{algorithmic}
\usepackage[switch]{lineno}

\usepackage{amsfonts}
\usepackage{amssymb}
\usepackage{xspace}
\usepackage{braket}
\usepackage{nicefrac}

\usepackage{enumitem}

\usepackage{cleveref}
\usepackage[basic,small,classfont=bold]{complexity}

\usepackage{tikz}
\usetikzlibrary{shapes,shapes.geometric,arrows,fit,calc,positioning,automata,chains,matrix.skeleton, arrows.meta,calc}

\tikzset{nature/.style={draw,rectangle}}
\tikzset{>={Stealth[scale=1.2]}}

\usepackage[symbol]{footmisc}

\newcommand{\ie}{\emph{i.e.}\xspace}
\newcommand{\eg}{\emph{e.g.}\xspace}

\newtheorem{example}{Example}

\newtheorem{theorem}{Theorem}

\newtheorem{lemma}{Lemma}
\newtheorem{definition}{Definition}

\newcommand{\dist}[1]{\Delta(#1)}
\newcommand{\bE}{\mathbb{E}}

\DeclareMathOperator*{\argmax}{arg\,max}
\newcommand{\given}{\!\mid\!}

\newcommand{\pib}{\pi_{B}}
\newcommand{\pii}{\pi_{I}}
\newcommand{\dataset}{\mathcal{D}}
\newcommand{\cnt}{\#}
\newcommand{\Nwedge}{N_{\wedge}}
\newcommand{\cU}{\mathcal{U}}

\newcommand{\tuple}[1]{\langle #1 \rangle}

\newcommand{\Afree}{\bar{A}}

\newcommand{\supp}[2]{\mathrm{supp(#1,#2)}}

\let\Pr\relax
\DeclareMathOperator{\Pr}{\mathrm{Pr}}

\newcommand{\aVal}{\mathbf{aVal}}
\newcommand{\cVal}{\mathbf{cVal}}

\newcommand{\BibTeX}{B\kern-.05em{\sc i\kern-.025em b}\kern-.08em\TeX}

\begin{document}


\begin{frontmatter}


\paperid{6125} 


\title{Data-Efficient Safe Policy Improvement Using Parametric Structure}


\author[A]{\fnms{Kasper}~\snm{Engelen}\orcid{0000-0001-8986-9949}
\thanks{Corresponding Author. Email: kasper.engelen@uantwerpen.be}\footnote{Equal contribution, authors are listed alphabetically.}}
\author[A]{\fnms{Guillermo}~\snm{A. P\'{e}rez}\orcid{0000-0002-1200-4952}\footnotemark
}
\author[A]{\fnms{Marnix}~\snm{Suilen}\orcid{0000-0003-2163-3504}
\footnotemark
} 

\address[A]{University of Antwerp -- Flanders Make}


\begin{abstract}
Safe policy improvement (SPI) is an offline reinforcement learning problem in which a new policy that reliably outperforms the behavior policy with high confidence needs to be computed using only a dataset and the behavior policy.
Markov decision processes (MDPs) are the standard formalism for modeling environments in SPI.
In many applications, additional information in the form of \emph{parametric dependencies} between distributions in the transition dynamics is available.
We make SPI more data-efficient by leveraging these dependencies through three contributions: 
(1) a \emph{parametric} SPI algorithm that exploits known correlations between distributions to more accurately estimate the transition dynamics using the same amount of data;
(2) a preprocessing technique that \emph{prunes} redundant actions from the environment through a game-based abstraction;
and (3) a more advanced preprocessing technique, based on \emph{satisfiability modulo theory} (SMT) solving, that can identify more actions to prune.
Empirical results and an ablation study show that our techniques increase the data efficiency of SPI by multiple orders of magnitude while maintaining the same reliability guarantees.
\end{abstract}

\end{frontmatter}


\section{Introduction}

\emph{Reinforcement learning} (RL) is the standard paradigm for solving 
sequential decision-making problems in unknown environments~\cite{DBLP:books/lib/SuttonB98}.
In an \emph{online} RL setting, the agent interacts with an unknown environment to explore and gather information about its dynamics.
This environment is typically modeled as a Markov decision process (MDP), and the agent's goal is to compute a policy that maximizes the expected cumulative discounted reward~\cite{DBLP:books/wi/Puterman94}. 

Such online interaction is often undesirable or dangerous in real-world application areas such as robotics or healthcare~\cite{DBLP:journals/corr/abs-2005-01643,DBLP:books/sp/12/Kober012}.
Online \emph{safe} RL methods attempt to mitigate these concerns by constraining the agent's interaction~\cite{DBLP:journals/jmlr/GarciaF15}, for instance, through shielding~\cite{DBLP:conf/aaai/AlshiekhBEKNT18} or constrained actor-critic methods~\cite{DBLP:conf/aaai/YangSTS21}.
Nevertheless, online safe RL approaches still fundamentally rely on the agent exploring its environment.

In settings where such exploration is infeasible, \emph{offline} RL can serve as an alternative.
Offline RL prevents the agent from interacting with the environment and instead only allows the use of a dataset of past interactions generated by some \emph{behavior policy}~\cite{DBLP:journals/corr/abs-2005-01643,DBLP:books/sp/12/LangeGR12}.
Using only this dataset, the agent is tasked with finding a policy that performs well.
In settings where the behavior policy is also available, the task is to find a new policy that outperforms the behavior policy. 

\emph{Safe policy improvement} (SPI)~\cite{DBLP:conf/icml/ThomasTG15,DBLP:conf/nips/GhavamzadehPC16} is an offline RL problem that imposes a probabilistic correctness guarantee on the performance gain of the new policy over the behavior policy.
Safe policy improvement with baseline bootstrapping (SPIBB)~\cite{DBLP:conf/icml/LarocheTC19} is one of the most commonly used SPI methods.
SPIBB ensures the improvement guarantee is satisfied by constraining the standard policy iteration algorithm to \emph{bootstrap} to (\ie, follow) the behavior policy in state-action pairs where insufficient data is available.
Note that most SPI methods that ensure the improvement guarantee rely on discrete environments~\cite{DBLP:conf/nips/GhavamzadehPC16,DBLP:conf/icml/LarocheTC19,DBLP:conf/ijcai/WienhoftSSDB023,DBLP:conf/icml/Castellini0ZSFS23,DBLP:conf/icml/0002ZCSSF24,DBLP:conf/icaart/SchollDOU22,DBLP:conf/aaai/SimaoS023}.

Since the agent cannot collect any new data about the environment in SPI, it is paramount to make the most of the available information.
Existing SPI methods scale poorly in terms of the number of samples required, as also evidenced by the practice of using hyperparameters instead of the theoretically required sample sizes, see, \eg,~\cite{DBLP:conf/icml/LarocheTC19,DBLP:conf/ijcai/WienhoftSSDB023}. 
This challenge to make SPI(BB) more data efficient has given rise to methods that exploit structure in the underlying MDP, such as factored state-spaces~\cite{DBLP:conf/aaai/SimaoS19,DBLP:conf/ijcai/SimaoS19} or graph connectivity~\cite{DBLP:conf/ijcai/WienhoftSSDB023}.

While exact probabilities are often unknown, many real-world problems exhibit structure and dependencies in the underlying probability distributions.
For instance, in robotics, probability distributions are used to represent many different sources of uncertainty~\cite{DBLP:journals/trob/LauriHP23}.
One such source could be component failure, and the same component may fail with the same probability distribution across several different configurations, \ie, state-action pairs, and thus introduce dependencies between the probability distributions.
For a concrete example, consider slippery grid worlds such as Frozen Lake, where a single slip probability affects multiple transitions~\cite{DBLP:journals/corr/abs-2407-17032}.

\emph{Parametric MDPs} (pMDPs) are an extension of MDPs suitable to model parametric dependencies between probability distributions~\cite{DBLP:conf/colt/GopalanM15,DBLP:journals/iandc/BaierHHJKK20,DBLP:conf/birthday/0001JK22}.
In a pMDP, each transition is assigned a polynomial over parameters that can occur on multiple transitions, effectively defining dependencies between probability distributions on different state-action pairs.

\subsection*{Contributions}

In this paper, we investigate how to improve the data efficiency of SPI methods using such parametric structures.
We present three methods that exploit the additional knowledge given by the pMDP to make SPIBB more data-efficient.
Specifically, we introduce a parametric variant of SPIBB, as well as two preprocessing techniques that are independent of the dataset.
In more detail, our contributions are:

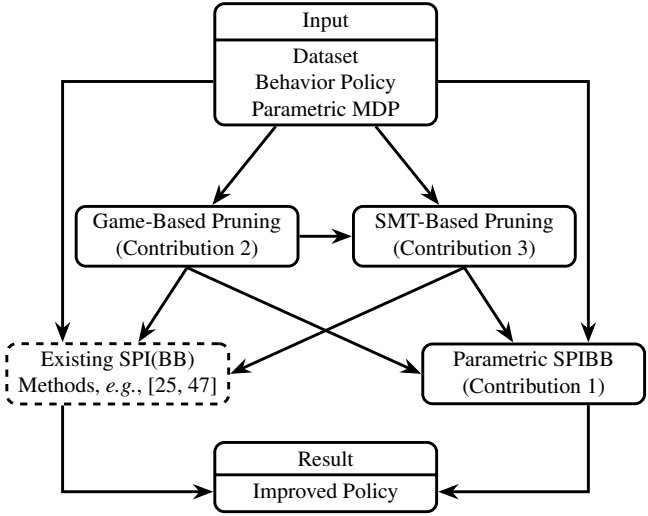
\begin{figure}[t]
    \centering
    \resizebox{1.0\linewidth}{!}{
    \begin{tikzpicture}[
squarenode/.style={rectangle, draw, rounded corners=4, very thick, minimum size=5mm, minimum width=3.2cm, minimum height=0.8cm, align=center},
squarenodedashed/.style={rectangle, dashed, draw, rounded corners=4, very thick, minimum size=5mm, minimum width=3.2cm, minimum height=0.8cm, align=center},
splitnode/.style={rectangle split,rectangle split parts=2, draw, rounded corners=4, very thick, minimum size=5mm, minimum width=3.2cm, minimum height=0.8cm, align=center},
->/.style = {-Stealth,very thick}%
]

\node[splitnode] at(0,0.5) (pMDP) {Input \nodepart{second} Dataset \\ Behavior Policy \\ Parametric MDP};

\node[squarenode] at(-2,-2) (gbpruning) {Game-Based Pruning \\ (Contribution 2)};
\node[squarenode] at(2,-2) (smtpruning) {SMT-Based Pruning \\ (Contribution 3)};

\node[squarenodedashed] at(-3,-4) (sSPIBB) {Existing SPI(BB) \\ Methods, \eg, \cite{DBLP:conf/icml/LarocheTC19,DBLP:conf/ijcai/WienhoftSSDB023}};
\node[squarenode] at(3,-4) (pSPIBB) {Parametric SPIBB \\ (Contribution 1)};

\node[splitnode] at(0,-5.5) (result) {Result \nodepart{second} Improved Policy};

\draw[->] (pMDP) -- (gbpruning);
\draw[->] (pMDP) -- (smtpruning);
\draw[->] ([yshift=-0.25cm]pMDP.west) -| (sSPIBB.150);
\draw[->] ([yshift=-0.25cm]pMDP.east) -| (pSPIBB.30);
\draw[->] (gbpruning) -- (smtpruning);
\draw[->] (gbpruning.south) -- (sSPIBB);
\draw[->] (gbpruning.south) -- (pSPIBB.west);
\draw[->] (smtpruning.south) -- (sSPIBB.east);
\draw[->] (smtpruning.south) -- (pSPIBB);
\draw[->] (sSPIBB.210) |- ([yshift=-0.5cm]result);
\draw[->] (pSPIBB.330) |- ([yshift=-0.5cm]result);

\end{tikzpicture}
    }
    \vspace{1em}
    \caption{Overview of our contributions and how they relate.}
    \label{fig:overview_flowchart}
\vspace{3em}
\end{figure}

\begin{enumerate}
    \item \textbf{Parametric SPIBB} (\Cref{sec:parametric_spibb})\textbf{.} 
    We extend SPIBB into \emph{parametric SPIBB} (pSPIBB) that exploits the parametric structure of the given pMDP.
    Specifically, when probability distributions at different state-action pairs follow the same parametric structure, we may lump together samples from these state-action pairs.
    This effectively increases the available data, and
    as a result, pSPIBB has fewer state-action pairs that need to be bootstrapped to the behavior policy compared to standard SPIBB.
    \item \textbf{Game-based  pruning} (\Cref{sec:game-based_analysis})\textbf{.} 
    Our second contribution is a preprocessing method that prunes state-action pairs from the environment via a game-based abstraction.
    By reasoning over all possible probability distributions, we obtain lower and upper bounds on the value of the pMDP.
    Using these bounds, we prune sub-optimal state-action pairs for which we can guarantee that an optimal policy will never visit these state-action pairs.
    \item \textbf{SMT-based pruning} (\Cref{sec:smt-based_analysis})\textbf{.}
    Our third contribution is an exact pruning method based on a \emph{satisfiability modulo theory} (SMT) query.
    Since game-based pruning is an abstraction technique, it can be overly conservative.
    In contrast, SMT-based pruning reasons directly over the transition polynomials, and thus potentially removes more state-action pairs than game-based pruning, but at additional computational expense.
\end{enumerate}

\Cref{fig:overview_flowchart} provides an overview of how our contributions relate.
In particular, standard SPI(BB) methods can be used on the input by simply ignoring the parametric MDP, while pSPIBB takes the parametric structure into account. 
Additionally, the two pruning methods serve as preprocessing steps for both standard SPI(BB) methods and pSPIBB.
In \Cref{sec:experiments}, we provide an extensive empirical evaluation and ablation study for these techniques on four well-known benchmarks and a newly introduced \emph{Rock-Paper-Scissors} environment.
All proofs can be found in the appendix.

\section{Background}
By $|X|$, we denote the number of elements in a finite set $X$.
A discrete probability distribution over a finite set $X$ is a function $\mu \colon X \to [0,1]$ with $\sum_{x \in X} \mu(x) = 1$.
The set of all distributions over $X$ is denoted $\dist{X}$.
The \emph{support} of a probability distribution $\mu$ is the set of all elements with non-zero probability: $\mathrm{supp}(\mu) = \{x \in X \mid \mu(x) > 0\}$.

\begin{definition}[MDP]
    A \emph{Markov decision process} (MDP) is a tuple $\tuple{S, \iota, A, T, R, \gamma}$, where $S$ is a finite set of states, $\iota \in S$ is the initial state, $A$ is a finite set of actions, $T \colon S \times A \to \dist{S}$ is the probabilistic transition function, $R \colon S \times A \to [-R_{\max},R_{\max}]$ is the (bounded) reward function, and $\gamma \in (0,1)$ is the discount factor.
\end{definition}
We write $T(s' \given s, a)$ for the probability $T(s,a)(s')$. 
By $\mathrm{supp}(s,a) = \left\{ s' \given T(s' \given s, a) > 0 \right\}$ we denote the \emph{support} of a state-action pair $\langle s,a \rangle$.
A \emph{trajectory} is a sequence of successive states and actions: $s_0a_0s_1\dots$, such that $s_{i+1} \in \mathrm{supp}(s_i,a_i)$ and $s_0 = \iota$.

A stationary (or memoryless) stochastic policy is a function $\pi \colon S \to \dist{A}$.
The objective is to find a policy that maximizes the expected cumulative discounted reward:
$\pi^* = \argmax_{\pi} \bE \left[ \sum\nolimits_{t=0}^{\infty} \gamma^t R(s_t,a_t) \right]$,
where $\tuple{s_t,a_t}$ is the state-action pair visited at time $t$ when following $\pi$.
The \emph{value} of an MDP is the unique least fixed point of the following recursive equations:
\begin{align}
    Q(s,a) &= R(s,a) + \gamma \sum\nolimits_{s' \in S} T(s' \given s,a )V(s), \label{eq:q_values}\\
    V(s) &= \max_{a \in A} Q(s,a), \label{eq:state_values}
\end{align}
where $Q(s,a)$ are called the state-action values and $V(s)$ are the state-values, both initialized at zero.
We write $V^\pi$ and $Q^\pi$ for both value functions under some policy $\pi$, which we call the \emph{performance} of $\pi$.
Optimal state and state-action values are given by the performance of an optimal policy $\pi^*$, denoted by $V^* = V^{\pi^*}$ and $Q^* = Q^{\pi^*}$, respectively.

\subsection{Safe Policy Improvement}
The offline RL problem of \emph{safe policy improvement} (SPI) is to compute a policy $\pii$ that improves over a given behavior policy $\pib$, based on historical data $\dataset = \braket{s_i,a_i,s_{i+1}}_{i = 0 \dots n}$ obtained from executing $\pib$.
Specifically, for an \emph{admissible performance loss} $\zeta$ and confidence parameter $\delta$, the performance of the improved policy $\pii$ must be better than that of $\pib$ with high probability:
\begin{equation}
\label{eq:spi_guarantee}
    \Pr\big( V^{\pii}(\iota) \geq V^{\pib}(\iota) - \zeta \big) > 1 - \delta.
\end{equation}

For a dataset $\dataset$, we denote by $\cnt_\dataset(s,a)$ and $\cnt_{\dataset}(s,a,s')$ the number of times a state-action pair or transition occurs in $\dataset$.
The \emph{maximum likelihood estimate} MDP (MLE-MDP) of a dataset $\dataset$ is given by $\tilde{M} = \tuple{S, \iota, A, \tilde{T}, R, \gamma}$, where 
\begin{equation}
\label{eq:mle_est}
\tilde{T}(s' \given s, a) = \frac{\cnt_{\dataset}(s,a,s')}{\cnt_{\dataset}(s,a)}
\end{equation}
when $\cnt_{\dataset}(s,a) > 0$, and zero otherwise.\footnote[1]{Note that when no data is available for a state-action pair, that action is \emph{disabled} at that state in the MLE-MDP.}

The \emph{safe policy improvement with baseline bootstrapping} (SPIBB)~\cite{DBLP:conf/icml/LarocheTC19} algorithm solves the SPI problem by only deviating from the behavior policy in states where sufficient data is available.
Specifically, SPIBB defines a set $\cU$ of state-action pairs that are \emph{uncertain}.
That is, state-action pairs of which the number of samples contained in $\dataset$ is below a specified threshold $\Nwedge$:
\begin{equation}
\label{eq:uncertainty_set}
\cU = \left\{ \braket{s,a} \in S \times A \mid \cnt_{\dataset}(s,a) < \Nwedge  \right\}.
\end{equation}

The improved policy $\pii$ is computed through policy iteration as a policy that chooses an action with the probability of $\pib$ if that state-action pair is uncertain, and concentrates the remaining probability mass on an action that is optimal for the MLE-MDP.
Formally, let $Q^{\pi}_{\tilde{M}}$ be the optimal state-action values of the MLE-MDP under a policy $\pi$ and let $\Afree(s) = \{a \in A \given \tuple{s,a} \not\in \cU\}$ denote the non-bootstrapped actions at state $s$, an improved bootstrapped policy $\pi'$ is computed as
\begin{align*}
   \pi'(a \given s) =
&\begin{cases}
        \pib(a \given s) & \forall \tuple{s,a} \in \cU, \\
        \sum\limits_{a' \in \Afree(s)
        }
        \pib(a' \given s) & a = \argmax\limits_{a' \in \Afree(s) 
        } Q^{\pi}_{\tilde{M}}(s,a'), \\
        0 & \text{ otherwise}.
    \end{cases}
\end{align*}
Iterating the process described above, starting from the behavior policy $\pib$, converges to the improved policy $\pii$~\citep{DBLP:conf/icml/LarocheTC19}.

For a fixed $\zeta$ and $\delta$, the threshold $\Nwedge$ required to ensure the SPI guarantee (\Cref{eq:spi_guarantee}) can be computed by a binary search~\cite{DBLP:conf/ijcai/WienhoftSSDB023} and is bounded by
\[
\Nwedge \leq \frac{32V_\mathit{max}^2}{\zeta(1-\gamma)^2} \log \frac{8|S|^2|A|^2}{\delta},
\]
where $V_\textit{max} \leq \nicefrac{R_{\max}}{1-\gamma} $ is a bound on the maximal value.
In practice, $\Nwedge$ is often set as a hyperparameter, together with $\delta$, and the \emph{possible} loss $\zeta$ is computed instead~\cite{DBLP:conf/icml/LarocheTC19}.
For a fixed $\Nwedge$ and $\delta$,~\cite{DBLP:conf/ijcai/WienhoftSSDB023} establish the currently best-known bound on $\zeta$ through the inverse incomplete Beta function~\cite{temme1992asymptotic} $I^{-1}$, being:
\[
\zeta \leq \frac{4V_\mathit{max}}{1-\gamma}\left(1-2I^{-1}_{\delta_T}\left( \frac{\Nwedge}{2}+1, \frac{\Nwedge}{2}+1 \right) \right) + c,
\]
where $\delta_T = \nicefrac{\delta}{2|S|^2|A|^2}$ and $c = V_{\tilde{M}}^{\pib}(\iota) - V_{\tilde{M}}^{\pii}(\iota)$ is a constant for the difference in performance between the behavior policy $\pib$ and the improved policy $\pii$ on the MLE-MDP $\tilde{M}$.

\subsection{Parametric MDPs}
\emph{Parametric MDPs} extend ordinary MDPs by a \emph{parametric} transition function that maps transitions $\tuple{s,a,s'}$ to polynomials with rational coefficients over a set of variables $X$. 
We write $\mathbb{Q}[X]$ for the set of all such polynomials.

\begin{definition}[pMDP]
A \emph{parametric MDP} (pMDP) is a tuple $\tuple{S, \iota, A, P, X,R, \gamma}$, where $S$, $\iota$, $A$, $R$, and $\gamma$ are as in MDPs, $X$ is a set of parameters and $P \colon S \times A \times S \to \mathbb{Q}[X]$ is the \emph{parametric} transition function.
\end{definition}

We assume that all transition polynomials that are not (syntactically) $0$ have some nonzero value in their image. 
We write $\mathrm{supp}(s,a) = \left\{ s' \given P(s,a,s') \neq 0 \right\}$.
A pMDP can be \emph{instantiated} into an MDP by applying a \emph{valuation} $\theta \in \mathbb{R}^{\lvert X \rvert}$, which results in an ordinary transition function $P_\theta (s' \given s, a) = P(s,a,s')(\theta)$.
We extend state and state-action values to pMDPs instantiated by a valuation $\theta$, denoted $V_\theta$ and $Q_\theta$, respectively, as well as $Q^\pi_\theta$ and $V^\pi_\theta$ for the performance of a policy $\pi$. 
%
We only consider \emph{graph-preserving} valuations $\theta \in \Theta_{\mathrm{gp}}$, \ie, valuations that all share the same underlying graph. 
Formally, all $\theta \in \Theta_{\mathrm{gp}}$ conform to the following, for all $s,s' \in S$ and all $a \in A$:
(1) $\sum_{s'' \in S} P_\theta(s, a, s'') = 1$, (2) $0 \leq P_\theta(s, a, s') \leq 1$, and (3) $P(s, a, s') \neq 0 \implies P_\theta(s, a,s') \neq 0$.

\section{Parametric SPIBB}
\label{sec:parametric_spibb}

As the introduction outlines, our first contribution is the extension of SPIBB into \emph{parametric SPIBB} (pSPIBB), which relies on the following two modifications.
First, we adapt the MLE-MDP (\Cref{eq:mle_est}) to consider occurrences of polynomials instead of transitions in the dataset.
Second, we replace the set of uncertain state-action pairs (\Cref{eq:uncertainty_set}) by again counting the occurrences of polynomials instead of state-action pairs.
We discuss these modifications in detail in the remainder of this section.

We assume that state-action distributions do not label transitions with identical polynomials, \ie, $P(s,a,s') = P(s,a,s'')$ implies $s' = s''$. 
This assumption is without loss of generality as successors along a transition with identical labels can be replaced by a fresh intermediate state followed by a uniform distribution over the original successors (cf. reduction to two-successor MDPs~\cite{DBLP:conf/ijcai/WienhoftSSDB023}). 
We lift the polynomial labeling to state-action pairs by setting $P(s,a) = \{P(s,a,s') \mid s' \in S\}$. 

We define equivalence classes of state-action pairs and transitions and treat the polynomials as labels to count their occurrences together in the data, a technique known as \emph{parameter tying}~\citep{DBLP:journals/corr/abs-2408-03093,DBLP:conf/icml/PoupartVHR06,DBLP:conf/qest/PolgreenWHA16}. 
Note that we do not need to know the exact transition probabilities, since we are only comparing the polynomial labels of the given pMDP.
We write $[s,a]_P$ for the set $\{\tuple{q,b} \given P(s,a) = P(q,b)\}$ and $[s,a,s']_P$ for $\{\tuple{q,b,q'} \given P(s,a,s') = P(q,b,q') \text{ and } P(s,a) = P(q,b)\}$. 
Note that equivalence between transitions also requires the state-action pair to be equivalent in terms of polynomial labels over successor transitions (see \Cref{fig:example_mle_fail}). 
We define the MLE-MDP over a pMDP by replacing \Cref{eq:mle_est} by:
\begin{equation}
\label{eq:mle_param_sharing}
\tilde{P}(s' \given s, a) = \frac{\sum_{(q,b,q') \in [s,a,s']_P} \cnt_{\dataset}(q,b,q')}{\sum_{(q,b) \in [s,a]_P}\cnt_{\dataset}(q,b)}.
\end{equation}

\begin{figure}[t]
\centering
    \begin{tikzpicture}[shorten >=1pt,auto,node distance=.9 cm, scale = 1.0, transform shape, ->]
		\tikzstyle{action} = [fill=black, shape=rectangle, draw]
		
		\node[state,initial=left,initial text={}](s0){$s_0$};
		\node[state, left = 2cm of s0, draw=none](invis){};
		\node[state, right = 2cm of s0, draw=none](invis2){};

		\node[action](s0_a)[above= of invis]{};
		\node[action](s0_b)[above= of s0]{};
		\node[action](s0_c)[above= of invis2]{};

		\node[state, above left=1cm and 0.2cm of s0_a](s1){$s_1$};
		\node[state, above right=1cm and 0.2cm of s0_a](s2){$s_2$};
		
		\node[state, above left=1cm and 0.7cm of s0_b](s3){$s_3$};
		\node[state, above= of s0_b](s4){$s_4$};
		\node[state, above right=1cm and 0.7cm of s0_b](s5){$s_5$};
		
		\node[state, above left=1cm and 0.2cm of s0_c](s6){$s_6$};
		\node[state, above right=1cm and 0.2cm of s0_c](s7){$s_7$};
		
		\draw[auto]
            (s0) edge node {$a$} (s0_a)
		(s0) edge node {$b$} (s0_b)
		(s0) edge node[swap] {$c$} (s0_c)
		(s0_a) edge node {$x$} (s1)
		(s0_a) edge node[swap] {$1-x$} (s2)
		(s0_b) edge node {$x$} (s3)
		(s0_b) edge node {$y$} (s4)
		(s0_b) edge[pos=0.3] node[swap] {$1-x-y$} (s5)
		(s0_c) edge node[pos=0.7] {$y$} (s6)
		(s0_c) edge node[swap] {$1-y$} (s7);
	\end{tikzpicture}
    \vspace{1em}
    \caption{
    An example pMDP with parameters $x$ and $y$.
    In this pMDP, pSPIBB's MLE-MDP does not reduce uncertainty, \eg, $P(s_0,a,s_1) = P(s_0,b,s_3)$ but $\tuple{s_0,b,s_3} \not\in [s_0,a,s_1]_P$ because $P(s_0,b,s_5) \not\in [s_0,a]_P$. 
    Further note that if $\nicefrac{\cnt_{\dataset}(s_0,c,s_6)}{\cnt_\dataset(s_0,c)} = \nicefrac{\cnt_{\dataset}(s_0,a,s_1)}{\cnt_\dataset(s_0,a)} = 1$, it is unclear what the values of $\tilde{P}(s_i \given s_0, b)$ should be set to in order to minimize uncertainty while obtaining a valid distribution.}
    \vspace{3em}
    \label{fig:example_mle_fail}
\end{figure}
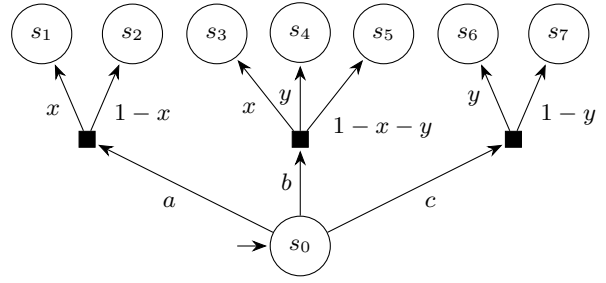

We modify SPIBB into pSPIBB by taking shared parameters into account.
We first replace \Cref{eq:mle_est} with \Cref{eq:mle_param_sharing}. 
Then we redefine the uncertainty set by taking into account not only how many times a state-action pair $\tuple{s,a}$ occurs in the dataset but also how many identically-labeled pairs $\tuple{q,b} \in [s,a]_L$ do:
\begin{equation}
\label{eq:pSPIBB_uncertainty_set}
\cU = \Big\{ \braket{s,a} \in S \times A \:\Big|\: \sum\nolimits_{(q,b) \in [s,a]_L}\cnt_{\dataset}(q,b) < \Nwedge  \Big\}.
\end{equation}
As a result, this new uncertainty set is a subset of the original uncertainty set that does not use parameter sharing (\Cref{eq:uncertainty_set}).
Note that this modification does not change the underlying correctness proof of SPIBB as we are simply merging two identical Bernoulli experiments into one.
Hence, pSPIBB is guaranteed to perform equally well or better compared to standard SPIBB while achieving the same improvement guarantee.

\section{Game-Based Analysis}
\label{sec:game-based_analysis}

Our second contribution enhances the data efficiency of both SPIBB and pSPIBB through a preprocessing step that prunes state-action pairs that an optimal policy never visits.
First, we analyze sequential decision-making in (p)MDPs from a game-based perspective and then show how to use a game-based abstraction scheme to prune unnecessary state-action pairs.

Decision-making under uncertainty can be phrased as a two-player game between an agent and \emph{nature}~\cite{DBLP:journals/jcss/Papadimitriou85}. 
The agent selects the actions, while nature draws successor states from the probability distributions in a (p)MDP. 
Taking this two-player game point of view, we can replace nature with \emph{cooperative} and \emph{antagonistic} players to obtain upper and lower bounds on the values that can be achieved in the unknown pMDP. 
Intuitively, playing against a cooperative player will result in an optimistic overapproximation of the value, while playing against an antagonistic player yields a worst-case underapproximation. 
These bounds hold regardless of the (unknown) transition probabilities that influence the behavior of nature. 
Hence, we can use the bound to prune the set of uncertain state-action pairs in SPI settings, even in the complete absence of data.

Playing against a cooperative or antagonistic nature, the agent's respective optimal policies are defined as:
\begin{align*}
    \overline{\pi} = {} & \argmax_\pi \Big(\sup_\rho \sum\nolimits_{t=0}^\infty \gamma^t R(s_t,a_t) \Big),\\
    \underline{\pi} = {} & \argmax_\pi \Big(\inf_\rho \sum\nolimits_{t=0}^\infty \gamma^t R(s_t,a_t) \Big),
\end{align*}
where the supremum and infinum range over trajectories $\rho = s_0 a_0 s_1 \dots$ consistent with $\pi$, \ie,  $a_{i} \in \mathrm{supp}(\pi(s_i))$ for all $i \geq 0$, and $\tuple{s_t,a_t}$ is the state-action pair visited at time $t$ along the run. 
We now define the antagonistic and cooperative values of a pMDP.

\begin{definition}[aVal and cVal]
The \emph{antagonistic} and \emph{cooperative values} of the pMDP are the unique solutions to the corresponding Bellman equations~\cite[Theorem 5.1]{DBLP:journals/tcs/ZwickP96}:
\begin{align}
    \aVal(s) &= \max_{a \in A} \min_{s' \in \mathrm{supp}(s,a)} R(s,a) + \gamma \aVal(s'), \label{eqn:aval}\\
    \cVal(s) &= \max_{a \in A} \max_{s' \in \mathrm{supp}(s,a)} R(s,a) + \gamma \cVal(s').\label{eqn:cval}
\end{align}
We omit $\theta$ from the notation since the values depend only on the support of the probability distributions and not on the actual probabilities themselves.
We write $\aVal^\pi$ and $\cVal^\pi$ for the antagonistic and cooperative value functions, respectively, under policy $\pi$, \ie, the unique solution to:
\begin{align}
    \aVal^\pi(s) &= \min_{\substack{a \in \mathrm{supp}(\pi(s))\\s' \in \mathrm{supp}(s,a)}} R(s,a) + \gamma \aVal^\pi(s'),\label{eqn:aval-pol}\\
    \cVal^\pi(s) &= \max_{\substack{a \in \mathrm{supp}(\pi(s))\\s' \in \mathrm{supp}(s,a)}} R(s,a) + \gamma \cVal^\pi(s').\label{eqn:cval-pol}
\end{align}
Note that $\aVal^{\underline{\pi}}(\iota) = \aVal(\iota)$ and $\cVal^{\overline{\pi}}(\iota) = \cVal(\iota)$.
\end{definition}

For intuition with respect to \Cref{eqn:aval-pol} and \Cref{eqn:cval-pol}, note that under policy $\pi$, all remaining probabilistic choices are resolved antagonistically or cooperatively, respectively. 
This resolution can be seen as modifying the policy to choose the worst, respectively, best, actions only---instead of having a distribution over multiple actions.

We use these antagonistic and cooperative values to find \emph{sub-optimal} state-action pairs, \ie, pairs $\tuple{s,a}$ for which we can guarantee that an optimal policy $\pi^*$ will never select $a$ in state $s$.
In that case, we can \emph{prune} $\tuple{s,a}$ from the pMDP.

To establish the correctness of this approach, we first present two results that describe the relation between the expected discounted reward and the antagonistic values of a state. 
The first result, \Cref{lem:expected_value_bound_for_policy}, states that $\aVal$ and $\cVal$ are inclusive lower and upper bounds on the expected discounted reward $V_\theta$.
\begin{lemma}
\label{lem:expected_value_bound_for_policy}
Given a pMDP, the following inequalities hold for all $s \in S$ and all graph-preserving valuations $\theta \in \Theta_{\mathrm{gp}}$.
\[
 \mathbf{aVal}(s) \leq V_\theta(s) \leq \mathbf{cVal}(s).
\]
\end{lemma}
\noindent
The proof immediately follows from the definitions and properties of convex combinations.

The second result, \Cref{lem:almost_sure_strict_lower_bound}, provides an effective characterization of when a policy almost-surely, \ie, with probability one, obtains a cumulative reward that is strictly higher than the antagonistic value. 

Let $\mathcal{I} = \{\tuple{s,a,s'} \in S \times A \times S \given \aVal(s) < R(s,a) + \gamma \aVal(s') \}$ denote the set of transitions along which the antagonistic value is strictly improved. We write $\tau_{\mathcal{I}}$ for the \emph{first hit time} of $\mathcal{I}$, \ie, the random variable for the first time step where a transition from $\mathcal{I}$ is traversed. 
The first hit time is $\infty$ if such a transition is never traversed. 
Intuitively, we show that a policy almost-surely obtains a cumulative discounted reward that is strictly larger than the antagonistic value if and only if it is a worst-case optimal policy and it almost-surely improves the antagonistic value along some transition.

\begin{theorem}\label{lem:almost_sure_strict_lower_bound}
Let $\theta \in \Theta_\mathsf{gp}$. 
Then, there is a policy $\pi$ such that
\begin{equation}\label{eqn:as-more-than-aval}
\Pr^\pi_\theta\left(\sum\nolimits_{t=0}^\infty \gamma^t R(s_t,a_t) > \mathbf{aVal}(\iota)\right) = 1,
\end{equation}
if and only if there is a worst-case optimal policy $\pi$, \ie, $\aVal^\pi(\iota) = \aVal(\iota)$, with $\Pr^\pi_\theta\left(\tau_{\mathcal{I}} < \infty \right) = 1$.
\end{theorem}
Almost-sure finite-time hitting (also known as reachability) does not depend on concrete probability values but only on their support~\cite{DBLP:books/daglib/0020348,DBLP:conf/ijcai/BharadwajRPT17}. 
Since the antagonistic value is independent of $\theta$ (and can be computed, \eg, as a solution to \Cref{eqn:aval}), we get that \Cref{eqn:as-more-than-aval} is also independent of $\theta$ and can be checked on the given pMDP for a policy $\pi$. 

As an immediate consequence of \Cref{lem:almost_sure_strict_lower_bound}, we get that the lower bound from \Cref{lem:expected_value_bound_for_policy} is strict under some conditions that are easy to check, as stated in the following lemma.
\begin{lemma}\label{lem:strict-nonstrict}
    If there exist $\theta \in \Theta_{\mathrm{gp}}$ and a policy $\pi$ such that $\aVal^\pi(\iota) = \aVal(\iota)$ and $\Pr^\pi_\theta\left(\tau_{\mathcal{I}} < \infty \right) = 1$ then for all $\theta \in \Theta_{\mathrm{gp}}$ we have $\mathbf{aVal}(\iota) < V_\theta(\iota)$.
\end{lemma}
It is easy to check, using \Cref{eqn:aval}, that any policy satisfying $\aVal(s) = \min_{s' \in \supp{s}{a}} R(s,a) + \gamma \aVal(s')$ for all $s$ and all $a \in \mathrm{supp}(\pi(s))$ is worst-case optimal, \ie, $\aVal^\pi(\iota) = \aVal(\iota)$. 
Conversely, for any worst-case optimal policy $\pi$, we can make local modifications to obtain a second worst-case optimal policy $\pi'$ satisfying the equations. 
These observations suggest the following two-step procedure to determine whether the conditions of \Cref{lem:strict-nonstrict} hold:
\begin{enumerate}
    \item Remove all state-action pairs $\tuple{s,a}$ for which we have $\aVal(s) > \min_{s' \in \supp{s}{a}} R(s,a) + \gamma \aVal(s')$;
    \item In the resulting pMDP, determine whether a transition from $\mathcal{I}$ is hit with probability $1$ under some policy $\pi$~\cite[Algorithm 45]{DBLP:books/daglib/0020348}.
\end{enumerate}
\noindent
The first step yields a sub-pMDP in which all policies are worst-case optimal. 
Moreover, any worst-case optimal policy has a corresponding policy in this sub-MDP. 
The second step tells us whether a policy to hit the improving transition almost-surely exists. 
Since the second step is realized on the sub-pMDP obtained from step $1$, we get a positive answer if and only if the conditions from \Cref{lem:strict-nonstrict} hold true 
(see \Cref{sec:strict-nonstrict}).

\subsection{Game-Based Pruning}
We use our results on game values to reduce the number of state-action pairs in the pMDP. 
Leveraging the bounds from \Cref{lem:expected_value_bound_for_policy,lem:strict-nonstrict}, we compare the antagonistic value of a state with the cooperative values of its successor states.

\begin{theorem}[aVal-cVal pruning]\label{thm:aval-cval-pruning}
Let $\pi$ be a policy and $\tuple{s,a}$ be a state-action pair. 
If  $a \in \mathrm{supp}(\pi(s))$ and
\begin{equation}\label{eqn:prune}
\mathbf{aVal}(s) > \max_{s' \in \supp{s}{a}} R(s,a) + \gamma \mathbf{cVal}(s'),
\end{equation}
then for all $\theta \in \Theta_{\mathrm{gp}}$ we have $V_\theta(s) > V_\theta^\pi(s)$. 
Moreover, if $\aVal^\pi(s) = \aVal(s)$ and $\Pr^\pi_\theta\left(\tau_{\mathcal{I}} < \infty \given s_0 = s\right) = 1$, for some $\pi$, the non-strict version of \Cref{eqn:prune} suffices.
\end{theorem}

Essentially, \Cref{thm:aval-cval-pruning} states that the action $a$ will not be chosen in state $s$ by any optimal policy and can therefore be pruned. 
Formally, if the condition of the result is true then, by  \Cref{eqn:cval}, any policy $\pi$ that plays $a$ from $s$ has a performance $Q_\theta^\pi(s,a)$ that is strictly smaller than the (optimal) antagonistic value $\aVal(s)$. 
By \Cref{lem:expected_value_bound_for_policy}, it then follows that its performance $V^\pi_\theta(s)$ is suboptimal, irrespective of the valuation $\theta \in \Theta_{\mathrm{gp}}$. 
The second part follows from \Cref{lem:strict-nonstrict}.

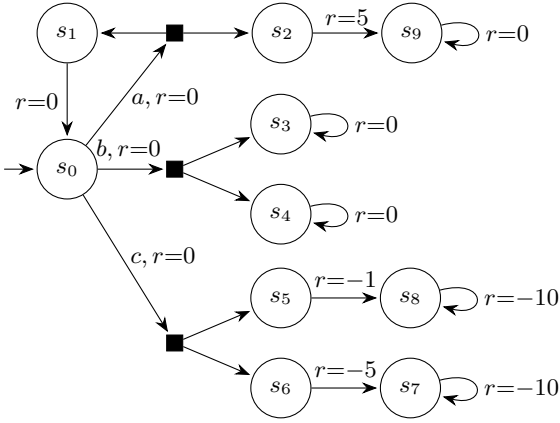
\begin{figure}[t]
\centering
    \begin{tikzpicture}[shorten >=1pt,auto,node distance=.9 cm, scale = 1.0, transform shape, ->]
        \tikzstyle{action} = [fill=black, shape=rectangle, draw]

        \node[state](s1){$s_1$};
        \node[state, below = 1cm of s1, initial=left, initial text={}](s0){$s_0$};
        \node[state, below = 1.5cm of s0, draw=none](invis){};

        \node[action](s0_a)[right= of s1]{};
        \node[action](s0_b)[right= of s0]{};
        \node[action](s0_c)[right= of invis]{};

        \node[state, right=of s0_a](s2){$s_2$};
        \node[state, above right=0.20cm and 1cm of s0_b](s3){$s_3$};
        \node[state, below right=0.20cm and 1cm of s0_b](s4){$s_4$};
        \node[state, above right=0.20cm and 1cm of s0_c](s5){$s_5$};
        \node[state, below right=0.20cm and 1cm of s0_c](s6){$s_6$};

        \node[state, right=of s6](s7){$s_7$};
        \node[state, right=of s5](s8){$s_8$};
        \node[state, right=of s2](s9){$s_9$};

        \draw (s0) edge [right, pos=0.45] node {$a, r{=}0$} (s0_a);
        \draw (s0) edge [above, pos=0.45] node {$b, r{=}0$} (s0_b);
        \draw (s0) edge [right, pos=0.45] node {$c, r{=}0$} (s0_c);

        \draw (s0_a) edge [above, pos=0.5]  (s1);
        \draw (s0_a) edge [above, pos=0.5] (s2);

        \draw[auto]
        (s0_b) edge (s3)
        (s0_b) edge [swap] (s4);
        
        \draw[auto]
        (s0_c) edge [pos=0.8] (s5)
        (s0_c) edge [swap] (s6);

        \draw (s1) edge [left, pos=0.5] node {$r{=}0$} (s0);
        \draw (s2) edge [pos=0.5] node {$r{=}5$} (s9);

        \draw (s3) edge [loop right] node {$r{=}0$} (s3);
        \draw (s4) edge [loop right] node {$r{=}0$} (s4);

        \draw (s5) edge [pos=0.5] node {$r{=}{-}1$} (s8);
        \draw (s6) edge [pos=0.5] node {$r{=}{-}5$} (s7);

        \draw (s7) edge [loop right] node {$r{=}{-}10$} (s7);
        \draw (s8) edge [loop right] node {$r{=}{-}10$} (s8);
        \draw (s9) edge [loop right] node {$r{=}0$} (s9);


    \end{tikzpicture}
    \vspace{1em}
    \caption{Example of a pMDP where both cases of aVal-cVal pruning apply (see \Cref{thm:aval-cval-pruning}).}
    \label{fig:example_prob_1}
    \vspace{3em}
\end{figure}

\begin{example}
    Consider the pMDP in \Cref{fig:example_prob_1} and let $\gamma = 0.95$. First, note that $\aVal(s_0) = \aVal(s_1)= 0$ and $\cVal(s_3) = \cVal(s_4) = 0$. If we apply the formula $\frac{r}{1-\gamma}$, we get $\cVal(s_7) = \aVal(s_7) = \cVal(s_8) = \aVal(s_8) = -200$, since $s_7$ and $s_8$ only depend on themselves. The states $s_5$ and $s_6$ depend on $s_8$ and $s_7$, respectively, and therefore $\cVal(s_5) = -191$ and $\cVal(s_6) = -195$.
    We thus have $\mathbf{aVal}(s_0) > -1 + \gamma \mathbf{cVal}(s_5)$ and $\mathbf{aVal}(s_0) > -5 + \gamma \mathbf{cVal}(s_6)$. 
    Therefore, by \Cref{thm:aval-cval-pruning}, the state-action pair $\tuple{s_0, c}$ can be removed. 
    Second, we have that the antagonistic value strictly improves along $\tuple{s_2, a, s_9}$, \ie, it is an element of $\mathcal{I}$, and it can be hit almost surely by playing action $a$ from $s_0$. 
    Therefore, the condition from the second part of \Cref{thm:aval-cval-pruning} is satisfied and $\tuple{s_0, b}$ can also be removed,
    since $\mathbf{aVal}(s_0) = 0 + \gamma \mathbf{cVal}(s_3)$ and $\mathbf{aVal}(s_0) > 0 + \gamma \mathbf{cVal}(s_4)$. 
\end{example}

\section{Parametric Value Function Analysis}\label{sec:smt-based_analysis}

In all techniques presented thus far, we have not yet fully exploited the parametric structure of pMDPs. 
Even in our definition of an MLE-MDP for pSPIBB, \Cref{eq:mle_param_sharing}, we do not yet consider all constraints implied by the sharing of parameters across transitions.

We proceed in this direction by leveraging the fact that both the state and state-action value functions $V_\theta$ and $Q_\theta$ are known to be rational functions of $\theta$~\cite{DBLP:conf/ictac/Daws04,DBLP:journals/fac/LanotteMT07}. 
For intuition, observe that the systems from \Cref{eq:state_values,eq:q_values} can be solved via algebraic manipulations, noting that probabilities are now polynomials over $\theta$. 
A complete algorithm to compute both functions is given in~\cite[Section 3]{DBLP:journals/iandc/BaierHHJKK20}. 

Instead of computing the rational functions, we encode the systems of polynomial constraints from \Cref{eq:state_values,eq:q_values} and use a \emph{satisfiability modulo theory} (SMT) solver to determine the (im)possibility of uncertain state-action pairs being useful to obtain $\pii$. 
In a nutshell, SMT solvers can check the satisfiability of first-order formulas over some \emph{theory}~\cite{DBLP:series/faia/BarrettSST21}. 
Here, the relevant theory is the \emph{existential theory of the reals} (ETR), which allows rational constants, multiplication, addition, and existentially quantified real-valued variables, see~\cite{DBLP:journals/corr/abs-2407-18006} for definitions and problems that can be encoded in ETR, including the encoding for pMDPs that we are alluding to.

\subsection{SMT-Based Pruning}
We now present a version of \Cref{thm:aval-cval-pruning} that compares 
the maximal value of a state over all possible valuations $\theta \in \Theta_{\mathrm{gp}}$ against its antagonistic value. 
Instead of the cooperative value, we now use a (sometimes) tighter upper bound on the value of the state to compare against the antagonistic value.

\begin{theorem}[aVal-Q pruning]\label{thm:aval-qval-pruning}
Let $\pi$ be a policy, $\theta \in \Theta_{\mathrm{gp}}$, and $\tuple{s,a}$ a state-action pair. 
If $a \in \mathrm{supp}(\pi(s))$ and
\begin{equation}\label{eqn:aval-q-prune}
\mathbf{aVal}(s) > Q_\theta(s,a),
\end{equation}
then $V_\theta(s) > V_\theta^\pi(s)$. 
Moreover, if $\aVal^\pi(s) = \aVal(s)$ and $\Pr^\pi_\theta\left(\tau_{\mathcal{I}} < \infty \given s_0 = s\right) = 1$, for some $\pi$, the non-strict version of \Cref{eqn:aval-q-prune} suffices.
\end{theorem}

The argument why the result holds is similar to the one for \Cref{thm:aval-cval-pruning}. 
If the condition holds,
by (the parametric versions of) \Cref{eq:q_values,eq:state_values}, all policies $\pi$ that play $a$ from $s$ have a performance $V_\theta^\pi(s)$ strictly smaller than $\aVal(s)$. 
Thus, by \Cref{lem:expected_value_bound_for_policy}, its performance $V^\pi_\theta(s)$ is suboptimal. 
The complement of \Cref{eqn:aval-q-prune}, with valuations $\theta$ quantified existentially and constraints on $\theta$ to be an element of $\Theta_{\mathrm{gp}}$, can be encoded in the existential theory of the reals.

Finally, observe that \Cref{thm:aval-qval-pruning} can be further improved by comparing the values of state-action pairs against each other instead of against the antagonistic value of their source state. 
Instead of precomputing $\aVal(s)$, we can replace the left-hand side of \Cref{eqn:aval-q-prune} by the (encoding of) the actual value of the state. 
Correctness of the result follows from (the parametric versions of) \Cref{eq:q_values,eq:state_values} directly.
\begin{theorem}[Q-Q pruning]\label{thm:qval-qval-pruning}
Let $\pi$ be a policy, $\theta \in \Theta_{\mathrm{gp}}$, and $\tuple{s,a}$ a state-action pair. If $a \in \mathrm{supp}(\pi(s))$ and
$V_\theta(s) > Q_\theta(s,a)$, then $V_\theta(s) > V_\theta^\pi(s)$.
\end{theorem}

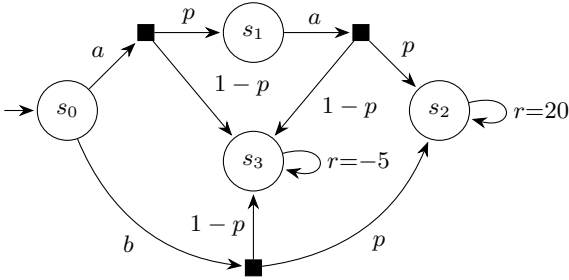
\begin{figure}[t]
\centering
    \begin{tikzpicture}[shorten >=1pt,auto,node distance=.9 cm, scale = 1.0, transform shape, ->]
        \tikzstyle{action} = [fill=black, shape=rectangle, draw]

        \node[state,initial left, initial text={}](s0){$s_0$};
        \node[action](s0_a)[above right= of s0]{};
        \node[state, right = of s0_a](s1){$s_1$};
        
        \node[action](s1_a)[right= of s1]{};
        \node[state, below right = of s1_a](s2){$s_2$};
        
        \node[state, below = of s1](s3){$s_3$};

        \node[action](s0_b)[below= of s3]{};

        \draw[auto] (s0) edge node {$a$} (s0_a);
        \draw[auto] (s0) edge [bend right,swap] node {$b$} (s0_b);
        \draw (s1) edge [above, pos=0.45] node {$a$} (s1_a);

        \draw[auto] (s0_a) edge node {$p$} (s1)
        (s0_a) edge[pos=0.65] node {$1-p$} (s3);

        \draw[auto] (s0_b) edge [swap,bend right] node {$p$} (s2)
        (s0_b) edge node {$1-p$} (s3);
        
        \draw[auto] (s1_a) edge  node {$1-p$} (s3)
        (s1_a) edge node {$p$} (s2);

        \draw (s2) edge [loop right] node {$r{=}20$} (s2);
        \draw (s3) edge [loop right] node {$r{=}{-}5$} (s3);
    \end{tikzpicture}
    \vspace{1em}
    \caption{Example of a pMDP where aVal-Q pruning fails but Q-Q pruning does apply (see \Cref{thm:aval-qval-pruning} and \Cref{thm:qval-qval-pruning}).}
    \label{fig:example_nwr_smt}
    \vspace{3em}
\end{figure}

\begin{example}
Consider the pMDP in \Cref{fig:example_nwr_smt} and let $\gamma = 0.9$. 
Since states $s_2$ and $s_3$ depend only on themselves, we can compute the worst case value as $\frac{r}{1-\gamma}$. We have $\aVal(s_2) = 20/0.1 = 200$ and $\aVal(s_3) = -5/0.1 = -50$. The $\aVal$ of the other states depends on the worst-case successor, and thus $\aVal(s_1)$ depends on $\aVal(s_3)$ and thus $\aVal(s_1) = 0 + 0.9(-50) = -45$. Similarly, $\aVal(s_0)$ also depends on $\aVal(s_3)$ and thus $\aVal(s_0) = 0 + 0.9(-50) = -45 < Q_\theta(s_0,a), Q_\theta(s_0,b)$, for all $\theta \in \Theta_{\mathrm{gp}}$, since playing either action yields some positive probability of reaching state $s_2$, where the expected discounted reward is positive. 
Hence, we cannot use \Cref{thm:aval-qval-pruning} to remove any action. 
By the same reasoning, however, we have $Q_\theta(s_0,a) \geq p^2$ and $p \geq Q_\theta(s_0,b)$. 
Since $p > p^2$ for all graph-preserving valuations, \ie, for all $p \in (0,1)$, then $V_\theta(s_0) = Q_\theta(s_0,b) < Q_\theta(s_1,a)$, thus by \Cref{thm:qval-qval-pruning}, we can remove $\tuple{s_0,a}$.
\end{example}

\begin{figure*}[ht!]
\centering 
\includegraphics[width=0.65\textwidth]{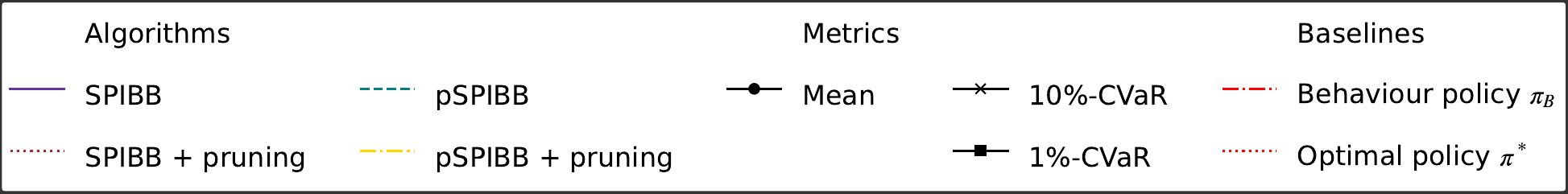}\\
\begin{subfigure}[b]{0.33\textwidth}
\centering
\includegraphics[width=\textwidth]{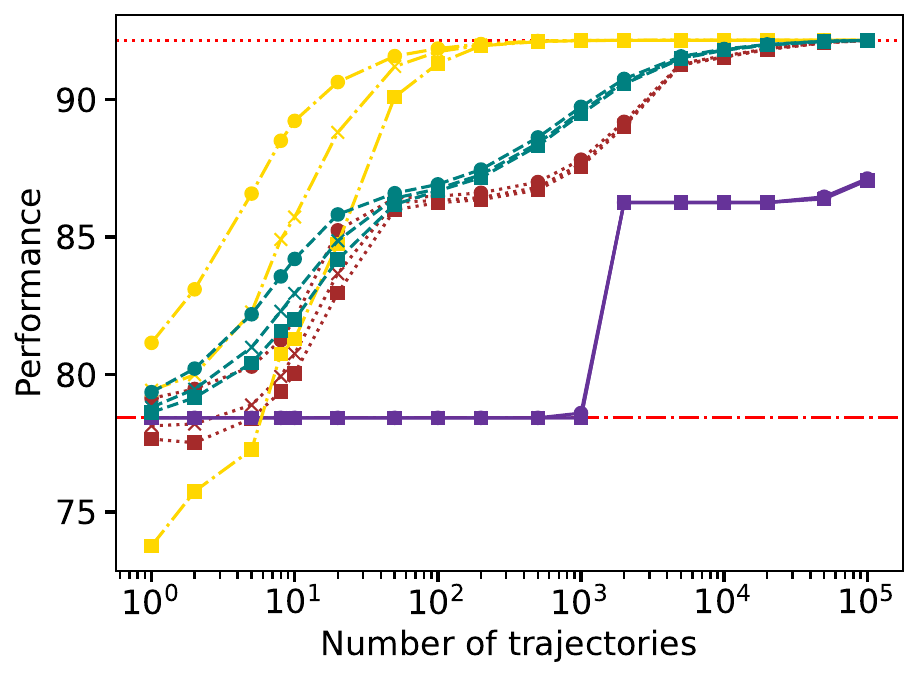}
\caption{Taxi, $\Nwedge=200$.}
\label{fig:experiments:taxi}
\end{subfigure}
\hfill
\begin{subfigure}[b]{0.33\textwidth}
\centering
\includegraphics[width=\textwidth]{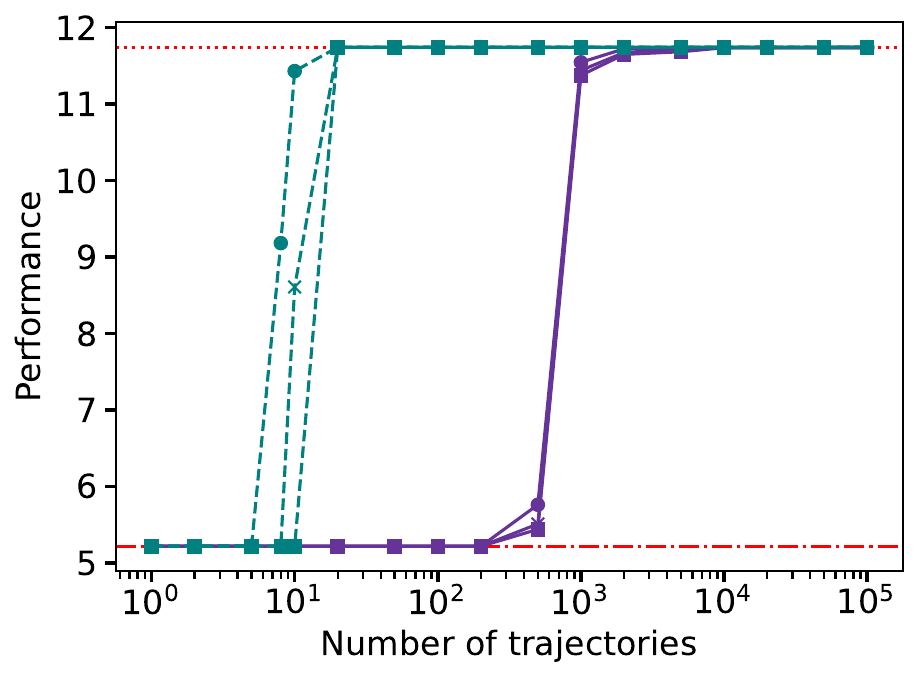}
\caption{Gridworld, $\Nwedge=200$.}
\label{fig:experiments:gridworld}
\end{subfigure}
\hfill
\begin{subfigure}[b]{0.33\textwidth}
\centering
\includegraphics[width=\textwidth]{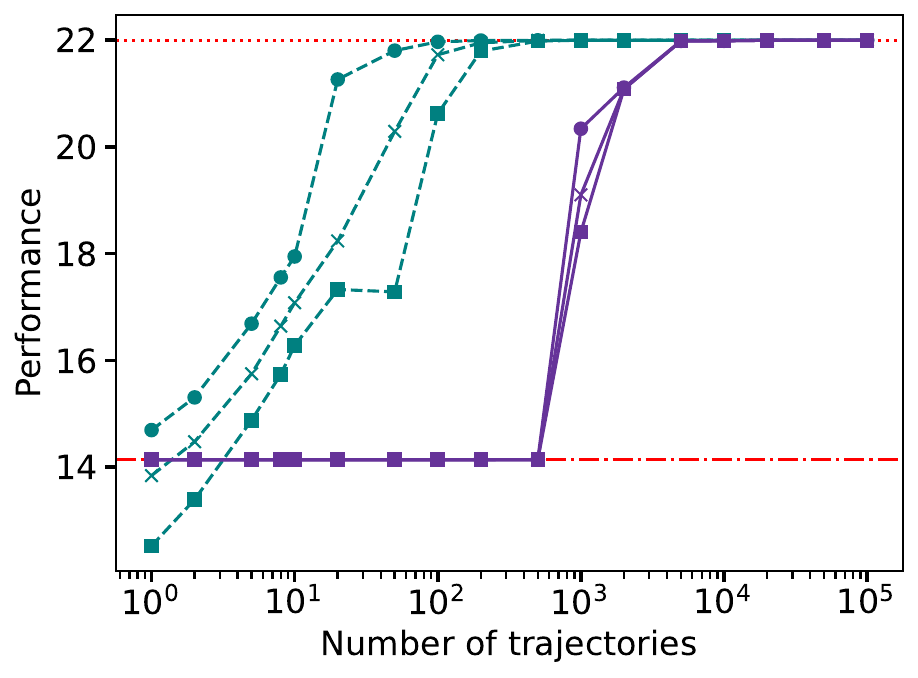}
\caption{Resource gathering, $\Nwedge=200$.}
\label{fig:experiments:resource}
\end{subfigure} \\
\vspace{2em}
\begin{subfigure}[b]{0.33\textwidth}
\centering
\includegraphics[width=\textwidth]{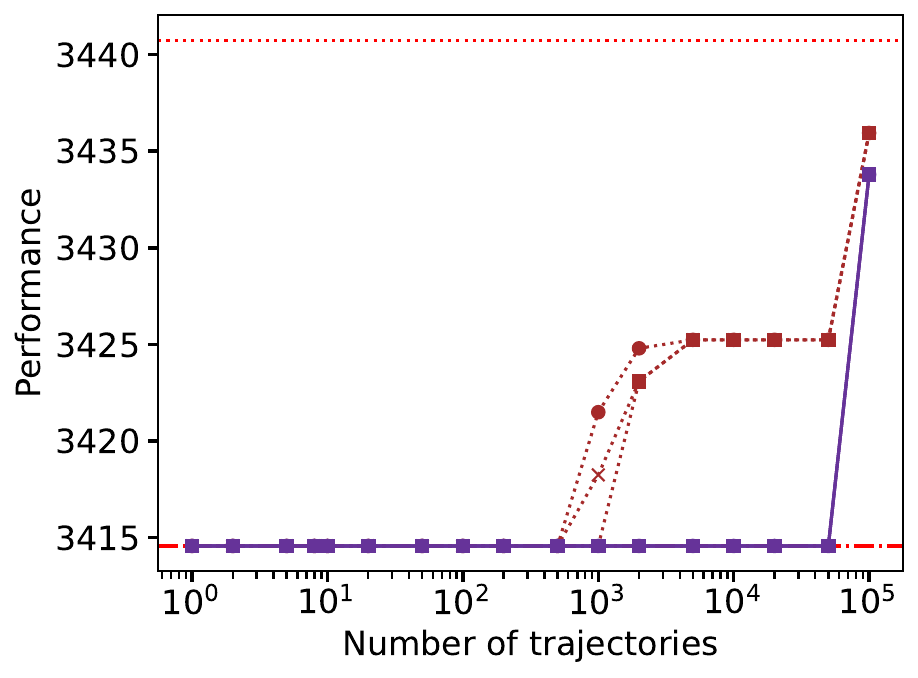}
\caption{Pac-Man, $\Nwedge=200$.}
\label{fig:experiments:pacman}
\end{subfigure}
\hfill
\begin{subfigure}[b]{0.33\textwidth}
\centering
\includegraphics[width=\textwidth]{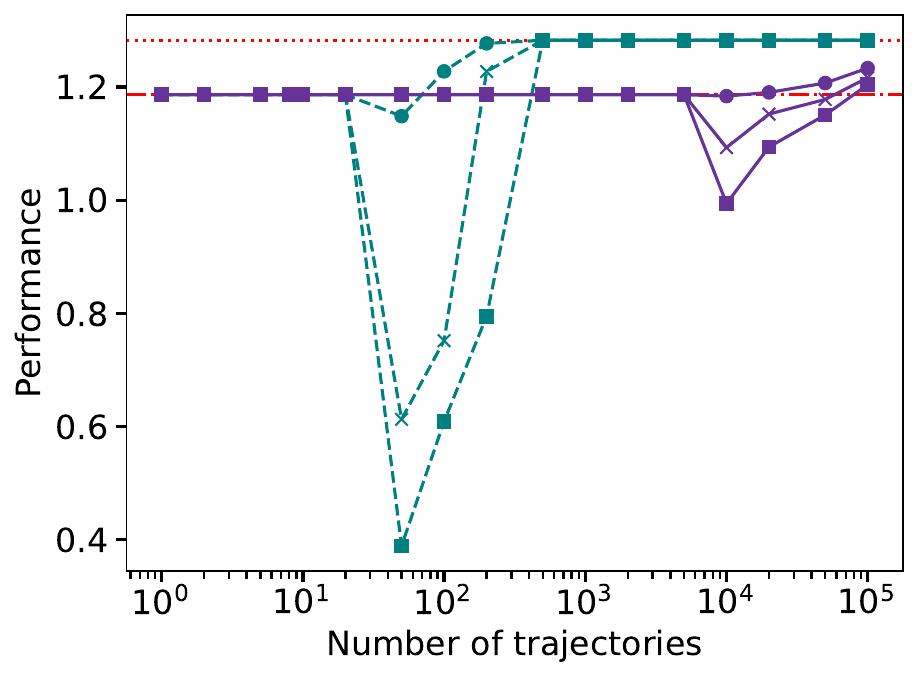}
\caption{RPS, $\Nwedge=200$.}
\label{fig:experiments:rps:200}
\end{subfigure}
\hfill
\begin{subfigure}[b]{0.33\textwidth}
\centering
\includegraphics[width=\textwidth]{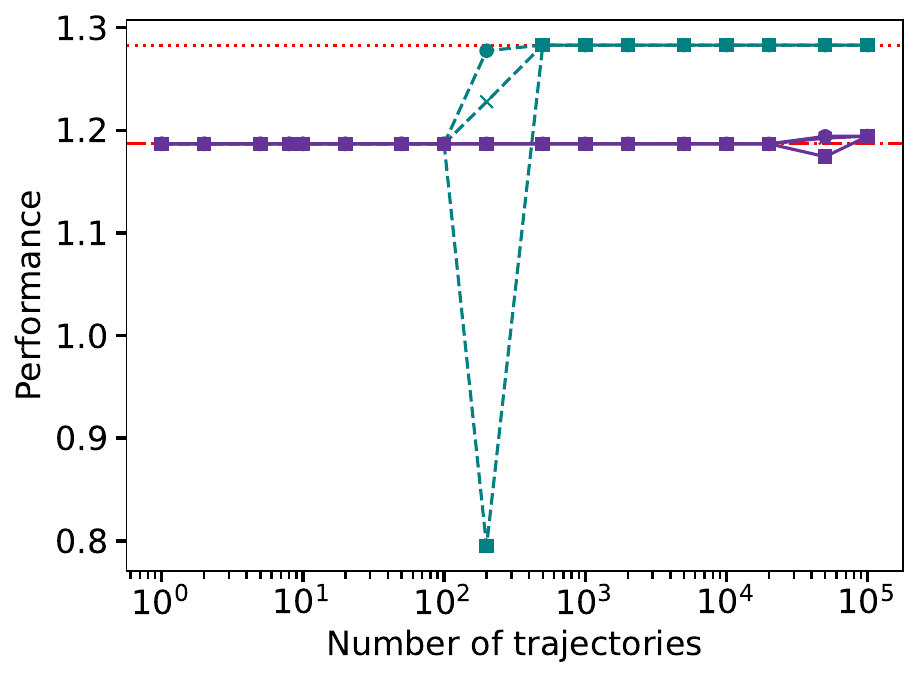}
\caption{RPS, $\Nwedge=1000$.}
\label{fig:experiments:rps:1000}
\end{subfigure}
\vspace{1em}
\caption{Comparison of methods on different environments for a specified $\Nwedge$ value.}
\label{fig:experiments}
\vspace{2em}  
\end{figure*}

\section{Experimental Evaluation}
\label{sec:experiments}
We empirically evaluate our contributions, pSPIBB, game-based pruning, and SMT-based pruning, compared to standard SPIBB~\cite{DBLP:conf/icml/LarocheTC19}.
The implementation and data is available at~\citep{experiments_zenodo}.

\paragraph{Environments.} 
We use the following standard benchmarks in our evaluation: \emph{Resource Gathering}~\cite{DBLP:conf/icml/BarrettN08}, \emph{Gridworld}~\cite{DBLP:conf/icml/LarocheTC19}, \emph{Pac-Man}~\cite{DBLP:conf/concur/0001KJSB20}, and \emph{Taxi}~\cite{DBLP:journals/jair/Dietterich00}.
Additionally, we implement a biased version of the \emph{Rock-Paper-Scissors} (RPS) game as a new benchmark.
In particular, the RPS benchmark models a $20$-round game where the agent plays against a stochastic opponent that plays with a bias the agent can exploit.
Full details on these environments can be found in \Cref{app:benchmak_overview}.

\subsection{Feasibility of SMT-Based Pruning}
\label{subsec:experiments:SMT:feasibility}
We start with evaluating the feasibility of SMT-based pruning.
Since SMT solving is known to be computationally expensive, we pose the research question:
\textbf{RQ1:} \emph{Is SMT-based pruning practically feasible?} 

\paragraph{Setup.}
We run SMT-based pruning on Gridworld to establish whether \emph{a single state-action pair} can be pruned, using the SMT-solvers \texttt{Z3}~\cite{DBLP:conf/tacas/MouraB08} and \texttt{cvc5}~\cite{DBLP:conf/tacas/BarbosaBBKLMMMN22}.
We set a timeout of $12$ hours. 
Both \texttt{Z3} and \texttt{cvc5} were invoked via their Python interfaces with default options. 
We did not use any advanced encodings or solver settings.

\paragraph{Results and Discussion.}
After a computation time of $6$ hours, \texttt{Z3} establishes that the state-action pair cannot be pruned, while \texttt{cvc5} timed out after $12$ hours.
From these results, we conclude that SMT-based pruning is (at present) practically infeasible, answering RQ1 negatively.
We emphasize, however, that our encoding of the conditions of \Cref{thm:aval-qval-pruning,thm:qval-qval-pruning} is quadratic in the number of states, actions, and transitions of the pMDP, and may be improved. 
Alternatively, tailor-made optimizations for these specific SMT instances could exist.

\subsection{Ablation of pSPIBB and Game-Based Pruning}
\label{subsec:experiments:ablation}

Next, we evaluate and compare the effects of our contributions regarding data efficiency.
Since we already established that SMT-based pruning is not practically feasible in \Cref{subsec:experiments:SMT:feasibility}, we omit this method from the comparison.
The central research question of this ablation study is
\textbf{RQ2:} \emph{Are pSPIBB, game-based pruning, and their combined use more data efficient than standard SPIBB?}  

\paragraph{Setup.}
We generate datasets of increasing size for each environment and compute the improved policy's performance across those datasets using each (combination of) our methods and standard SPIBB.
Whenever pSPIBB or game-based pruning has no (additional) effect because there are no parameters shared or state-action pairs to prune, we omit those curves from the plots. 
For all environments, we repeat the experiment for $1024$ different seeds.
We report the mean performance as well as the $10\%$ and $1\%$ conditional value at risk (CVaR)~\cite{rockafellar2000optimization}, which is the mean of the $10\%$ (resp. $1\%$) worst performing seeds.

\subsubsection{Results and Discussion}

\Cref{fig:experiments} shows the results obtained on all environments for $\Nwedge = 200$, with an additional experiment using $\Nwedge=1000$ for RPS.
The Taxi environment is the only benchmark that contains shared parameters and where game-based pruning removes state-action pairs.
There are no shared parameters in Pac-Man, and game-based pruning does not affect the other three environments.

\paragraph{Pruning and pSPIBB increase data efficiency. }
In \Cref{fig:experiments:taxi,fig:experiments:pacman} we see how pruning is an effective preprocessing method that can improve SPIBB's data efficiency by two orders of magnitude.
In the other environments, pruning had no effect.
Exploiting shared parameters as done in pSPIBB also increases data efficiency compared to standard SPIBB, as seen in \Cref{fig:experiments:taxi,fig:experiments:rps:200,fig:experiments:rps:1000,fig:experiments:gridworld,fig:experiments:resource}. In Gridworld and RPS we can see an improvement of more than two orders of magnitude in data efficiency.
The Pac-Man environment does not contain any shared parameters. 
Thus, pSPIBB coincides with SPIBB in this environment.
Finally, as showcased in \Cref{fig:experiments:taxi}, the combination of pSPIBB with game-based pruning delivered the greatest improvement in data efficiency, highlighting the benefits of combined usage of our contributions and positively answering RQ2.

\paragraph{Discussion of decreased performance. }
In Taxi, Resource Gathering, and RPS, a decrease in performance compared to the behavior policy occurs for some seeds, as evidenced by the $10\%$ and $1\%$-CVaR curves.
These decreases are more severe when using pSPIBB compared to standard SPIBB.
This behavior occurs because a wrong maximum likelihood estimate based on the data affects multiple state-action pairs simultaneously, exacerbating effects that may also occur in standard SPIBB.
Nonetheless, we emphasize the following two observations. 
First, increasing the $\Nwedge$ hyperparameter suppresses this effect while still maintaining the gains in data efficiency of our contributions, as evidenced by \Cref{fig:experiments:rps:200,fig:experiments:rps:1000}, and second, all observed decreases in performance are still within the allowed performance loss $\zeta$.
For instance, assuming an error $\delta=0.05$, we obtain in the Taxi environment an allowed loss $\zeta \approx 2895.02$ and in RPS we obtain $\zeta \approx 41.36$ and $\zeta \approx 19.09$ for $\Nwedge = 200$ and $1000$, respectively. 
All curves are well within the range of the behavior policy's performance $\pm \zeta$.

\section{Related Work}
In addition to the main SPI works mentioned in the introduction, significant work has been done to make SPI methods applicable in broader settings.
Most notably, multi-objective~\cite{DBLP:conf/nips/SatijaTPL21}, non-stationary~\cite{DBLP:conf/nips/ChandakJTWT20}, partially observable~\cite{DBLP:conf/aaai/SimaoS023}, and multi-agent~\cite{DBLP:conf/icml/0002ZCSSF24} settings.  
For a recent overview of SPI methods, we refer to~\cite{DBLP:conf/icaart/SchollDOU22}.

Since collecting more data is not an option, data efficiency is of great importance in offline RL settings such as SPI.
To that end, (offline) RL methods have been integrated with additional domain knowledge such as factored state spaces~\cite{DBLP:conf/aaai/SimaoS19} or (parametric) graph structures~\cite{DBLP:conf/ijcai/SimaoS19,DBLP:journals/corr/abs-2408-03093}, similar to our approach.
Alternatively, less conservative statistical methods to derive guarantees have been used in SPI~\cite{DBLP:conf/ijcai/WienhoftSSDB023} and more general learning settings~\cite{DBLP:journals/corr/abs-2404-05424}.
Another bottleneck of SPI(BB) is that most methods are tabular, \ie, they explicitly build an MLE-MDP from the dataset.
To alleviate this drawback,~\cite{DBLP:conf/icml/Castellini0ZSFS23} employ Monte Carlo Tree Search to enable SPIBB in large state spaces.

Game-based abstractions have been used in safety-constrained planning for partially observable MDPs~\cite{DBLP:conf/aaai/Chatterjee0PRZ17}.
This work was later extended to probabilistic constraints in~\cite{DBLP:conf/ijcai/ChatterjeeE0R18}.

\section{Conclusion}
We presented three new contributions that exploit parametric dependencies in the underlying environment to make safe policy improvement more data-efficient.
Our experimental evaluation showed that parametric SPIBB and game-based pruning (and their combination) are practical and efficient techniques, increasing data efficiency by two orders of magnitude.
In contrast, SMT-based pruning, while theoretically promising, is practically infeasible at present.

Future work will focus on increasing the computational efficiency of SMT-based pruning using more advanced encodings or other solvers.
In this direction, we highlight that our code is agnostic with respect to the actual SMT-solver used. 
Integrating additional SMT-solvers can be done by simply implementing the small SMT-solver interface we provided. 
New benchmarks can also be easily added.
Additionally, we aim to further increase the data-efficiency of SPI(BB) by integrating our methods with others such as structure learning~\cite{DBLP:conf/ijcai/SimaoS19}.
Another direction is to overapproximate the parametric distributions by intervals, reducing it to a robust MDP~\cite{DBLP:journals/mor/WiesemannKR13,DBLP:conf/birthday/SuilenBB0025}.
Such an approximation would increase computational efficiency at the cost of being more conservative, where the exact benefits of this trade-off remain an open question.



\begin{ack}
This work was supported by the FWO ``SynthEx'' project (G0AH524N).
We ran our experiments on the GPULab platform provided by IDLab (\url{https://gpulab.ilabt.imec.be/}).
We would also like to thank Nils Jansen and Sebastian Junges for their feedback on earlier versions of this work.
\end{ack}



\bibliography{references}

\clearpage
\newpage

\appendix
\onecolumn
\section{Proof of \Cref{lem:expected_value_bound_for_policy}}

Consider a pMDP $\tuple{S, \iota, A, P, X, R,  \gamma}$. 
Let $s \in S$ and $\theta \in \Theta_{\mathrm{gp}}$ be arbitrary. 
We will now prove that the following inequalities hold:
\[
 \aVal(s) \leq V_\theta(s) \leq \mathbf{cVal}(s).
\]
First, recall that:
\[
    V_\theta(s) = \sup_{\pi} \bE \left[ \sum_{t=0}^{\infty} \gamma^t R(s_t,a_t) \right],
\]
where the expectation is taken over the usual probability space, where trajectories are outcomes for Markov chains (indeed, having fixed $\pi$ in the MDP induced by the valuation $\theta$, we are dealing with a Markov chain). 
Further, recall that the expectation operator is monotonic. 
That is, for arbitrary random variables $X$ and $Y$, if $X \leq Y$ holds with probability $1$ and $\bE[X], \bE[Y]$ both exist then $\bE[X] \leq \bE[Y]$.

Now, let $\pi$ be an arbitrary policy and $\rho = s_0 a_0 s_1 \dots$ a trajectory consistent with $\pi$ such that $s_0 = s$. 
Observe that, by definition of the antagonistic and cooperative values, we have:
\[
    \aVal^\pi(s) \leq \sum^\infty_{t=0} \gamma^t R(s_t,a_t) \leq \cVal^\pi(s).
\]
The inequalities above hold \emph{surely} and therefore also hold \emph{almost surely}. 
Since the expected value of a constant is that same constant, by the above arguments, we get:
\[
    \aVal^\pi(s) = \bE[\aVal^\pi(s)] \leq \bE\left[\sum^\infty_{t=0} \gamma^t R(s_t,a_t)\right] \leq \bE[\cVal^\pi(s)] = \cVal^\pi(s).
\]
Finally, since $\pi$ was arbitrary, then by definition of supremum, we obtain:
\[
\sup_\pi\aVal^\pi(s) \leq \sup_\pi\bE\left[\sum^\infty_{t=0} \gamma^t R(s_t,a_t)\right] \leq \sup_\pi \cVal^\pi(s),
\]
as required.\qed

\section{Proof of \Cref{lem:almost_sure_strict_lower_bound}}
Consider a pMDP $\tuple{S, \iota, A, P, X, R, \gamma}$.
Let $\theta \in \Theta_\mathsf{gp}$. 
We will prove that there is a policy $\pi$ such that
\[
\Pr^\pi_\theta\left(\sum_{t=0}^\infty \gamma^t R(s_t,a_t) > \aVal(\iota)\right) = 1,
\]
if and only if there is a policy $\pi$ such that $\aVal^\pi(\iota) = \aVal(\iota)$ and $\Pr^\pi_\theta\left(\tau_{\mathcal{I}} < \infty \right) = 1$.

\subsection{Preliminaries}
We begin with a lemma that will greatly simplify our argument to establish the claim. 
The first part should be intuitive to the reader since it just says that every action played before (eventually) hitting $\mathcal{I}$ does not decrease the antagonistic value. 
The second part is less trivial: When a transition from $\mathcal{I}$ is traversed, the (local) antagonistic value may actually increase. 
Hence, a worst-case optimal policy $\pi$ may choose to play sub-optimally (locally) while still ensuring the best antagonistic value from $\iota$. 
Thankfully, the intuitive property that $\pi$ can always choose to play optimally from every reached state and not just the initial one does hold true.
\begin{lemma}\label{lem:cali}
    Let $\pi$ be a policy with $\aVal^\pi(\iota) = \aVal(\iota)$. 
    Consider a trajectory $s_0 a_0 s_1 \dots$ consistent with $\pi$ and write $\tau_{\mathcal{I}} \in \mathbb{N} \cup \{\infty\}$ for its first hit time to $\mathcal{I}$. 
    Then,
    \begin{equation}\label{eqn:before-cali}
        \forall T < \tau_{\mathcal{I}} : \sum_{t = 0}^{T} \gamma^t R(s_t,a_t) + \gamma^{T+1}\aVal^\pi(s_{T+1}) = \aVal^\pi(s_T) = \aVal(s_T).
    \end{equation}
    Moreover, there exists a second policy $\mu$ with $\aVal^\mu(\iota) = \aVal(\iota)$, $\Pr_\theta^\mu(\tau_{\mathcal{I}} < \infty) = \Pr_\theta^\pi(\tau_{\mathcal{I}} < \infty)$, and such that for all trajectories $s_0 a_0 s_1 \dots$ consistent with it have:
    \begin{equation}\label{eqn:strict-cali}
        \forall T \in \mathbb{N} : \sum_{t = T}^\infty \gamma^t R(s_t,a_t) \geq \aVal^\mu(s_T) = \aVal(s_T),
    \end{equation}
    with the inequality being strict if and only if $\tuple{s_T,a_T,s_{T+1}} \in \mathcal{I}$.
\end{lemma}
\begin{proof}
    The first part of the result follows from induction on the definition of $\aVal^\pi$, \ie,
    \[
        \aVal^\pi(s) = \inf_\rho \sum_{t=0}^\infty \gamma^t R(s_t,a_t),
    \]
    where $\rho$ ranges over trajectories consistent with $\pi$,
    \Cref{eqn:aval}, 
    and \Cref{eqn:aval-pol}. 
    
    For the second part, we modify $\pi$ to play in a worst-case optimal fashion even after the first transition from $\mathcal{I}$ is traversed. (Remember that a worst-case optimal policy for one state is not necessarily worst-case optimal for all states simultaneously. Instead, we will construct a policy that is also worst-case optimal for all states reached after $\mathcal{I}$.)
    That is, we change action choices so that $\aVal^\mu(s) = \aVal(s)$ holds for all $s \in S$. 
    This is an easy local change that is always possible since $\aVal$ is the unique solution to \Cref{eqn:aval} (cf.~\cite[Sec. 5]{DBLP:journals/tcs/ZwickP96}). 
    Note that it also preserves the probability of hitting $\mathcal{I}$ since the change is locally applied to actions chosen from states that are only ever reached from $\iota$ via transitions from $\mathcal{I}$ (otherwise, the first part of the claim would make it so that no change is actually needed). 
    Finally, the inequalities then follow from the definition of $\mathcal{I}$ and again an induction on \Cref{eqn:aval} and \Cref{eqn:aval-pol}.
\end{proof}

\subsection{Proof of the claim}
We begin with the reverse implication. 
That is, suppose there is a policy $\pi$ with $\aVal^\pi(\iota) = \aVal(\iota)$ such that $\Pr_\theta^\pi(\tau_{\mathcal{I}} < \infty) = 1$. 
Without loss of generality, we will also assume our policy enjoys the properties from $\mu$ in \Cref{lem:cali}. 
Now, consider an arbitrary trajectory $\rho = s_0 a_0 s_1 \dots$ consistent with $\pi$ and such that it eventually traverses a transition from $\mathcal{I}$. 
Write $\tau_{\mathcal{I}}$ for the first hit time of $\mathcal{I}$. 
The following hold:
\begin{align}
    \sum_{t=0}^\infty \gamma^t R(s_t,a_t) = {} & \sum_{t=0}^{\tau_{\mathcal{I}}-1} \gamma^t R(s_t,a_t) + \sum_{t=\tau_{\mathcal{I}}}^\infty \gamma^t R(s_t,a_t)\\
    {} > {} & \sum_{t=0}^{\tau_{\mathcal{I}}-1} \gamma^t R(s_t,a_t) + \gamma^{\tau_{\mathcal{I}}}\aVal(s_{\tau_{\mathcal{I}}}) & \text{by \Cref{lem:cali}, \Cref{eqn:strict-cali}}\\
    {} = {} & \aVal(s_0) & \text{by induction on \Cref{lem:cali}, \Cref{eqn:before-cali}}
\end{align}
To conclude, write $A$ for the event corresponding to having a trajectory with a value strictly larger than $\aVal(\iota)$ and $B$ for the event of having a finite hit time to $\mathcal{I}$. 
Recall that we assumed $\Pr_\theta^\pi(B) = 1$ and we just established that $\Pr_\theta^\pi(A \given B) = 1$. 
Hence, 
\[
    \Pr_\theta^\pi(A) \geq \Pr_\theta^\pi(A \cap B) = \Pr_\theta^\pi(B) \Pr_\theta^\pi(A \given B) = 1,
\]
and we get the desired result.

To prove the forward implication, we first argue that $\pi$ must be worst-case optimal. 
That is, we claim that $\aVal^\pi(\iota) = \aVal(\iota)$. 
Towards a contradiction, assume there is a trajectory $\rho = s_0 a_0 s_1 \dots$ consistent with $\pi$ such that $s_0 = \iota$ and $\sum_{t=0}^\infty \gamma^t R(s_t,a_t) < \aVal(s_0)$. 
The trajectory can be chosen so that there is some $T \in \mathbb{N}$ such that $\aVal(s_0) - \sum_{t=0}^{T-1}\gamma^t R(s_t,a_t) > \gamma^T\frac{R_{\max}}{1 - \gamma}$. 
This is because the prefixes of $\sum_{t=0}^{\infty}\gamma^t R(s_t,a_t)$ are converging towards a value that is strictly smaller than $\aVal(s_0)$ and after some large enough $T$, the gap becomes irrecoverable.
Hence, after the first divergence from the antagonistic value, discounting ensures such a $T$ indeed exists.
Consider the set of all trajectories consistent with $\pi$ and such that they have $s_0 a_0 s_1 \dots a_{T-1} s_T$ as a prefix. 
Note that the measure of this event, which we denote $\rho_T$ is nonzero, \ie,
\[
    \Pr_\theta^\pi(\rho_T) = \prod_{t=0}^{T-1} P_\theta(s_t,a_t,s_{t+1}).
\]
The fact that it is nonzero follows from our use of graph-preserving valuations only. 
Now, from our choice of trajectory and $T$, we get that all trajectories in $\rho_T$ are such that their discounted sum of rewards is strictly smaller than $\aVal(s_0)$. 
We get a contradiction with our assumption that almost all trajectories have a discounted sum of rewards that is, instead, strictly \emph{larger}.

We have proved that $\pi$ is worst-case optimal. 
It remains to argue that $\Pr_\theta^\pi(\tau_{\mathcal{I}} < \infty) = 1$.
Towards a contradiction, suppose some nonzero measure of trajectories never hits $\mathcal{I}$. 
By \Cref{lem:cali}, \Cref{eqn:before-cali}, we get that the discounted sum of rewards of such trajectories is exactly $\aVal(\iota)$. 
However, this means we get the same contradiction as above: we have a nonzero measure of trajectories whose discounted sum of rewards is not strictly larger than $\aVal(\iota)$, which goes against our initial assumption.\qed

\section{Proof of \Cref{lem:strict-nonstrict}}
Suppose $\aVal^\pi(\iota) = \aVal(\iota)$ and $\Pr^\pi_\theta\left(\tau_{\mathcal{I}} < \infty \right) = 1$ for some policy $\pi$, and recall that the concrete choice of $\theta$ is not relevant for the latter. 
Now, from the reverse implication of \Cref{lem:almost_sure_strict_lower_bound} we get that there is some policy $\pi$ such that:
\[
\Pr^\pi_\theta\left(\sum_{t=0}^\infty \gamma^t R(s_t,a_t) > \aVal(\iota)\right) = 1.
\]
It follows that:
\[
    V_\theta^\pi(\iota) = \bE\left[\sum^\infty_{t=0} \gamma^t R(s_t,a_t)\right] > \aVal(\iota),
\]
whence $\aVal(\iota) < V_\theta^\pi(\iota) \leq V_\theta(\iota)$ as required.\qed

\section{Algorithmically checking \Cref{lem:strict-nonstrict}}\label{sec:strict-nonstrict}
Recall the proposed algorithm is a two-step process:
\begin{enumerate}
    \item Remove all state-action pairs $\tuple{s,a}$ for which we have $\aVal(s) > \min_{s' \in \supp{s}{a}} R(s,a) + \gamma \aVal(s')$;
    \item In the resulting pMDP, use Algorithm 45 for almost-sure reachability from~\cite{DBLP:books/daglib/0020348}.
\end{enumerate}
Termination of the algorithm is trivial since Algorithm 45 also terminates. 
It remains for us to establish correctness. 
We proceed by establishing partial correctness and completeness.

\subsection{Partial correctness}
We start from the correctness guarantee of Algorithm 45 to determine whether there is a policy that ensures almost sure hitting of $\mathcal{I}$ and then establish that the same policy is also worst-case optimal.
Let $\pi$ be a policy such that $\Pr_\theta^\pi(\tau_{\mathcal{I}} < \infty) = 1$ in the sub-pMDP as found by Algorithm 45.

Note that the first step of the procedure can be seen as a restriction of the actions allowed from each state of the pMDP. 
We formalize this by writing $B(s) \subseteq A$ for the set of actions allowed from state $s \in S$. 
Now, consider any further restriction $A'$ such that $\emptyset \neq A'(s) \subseteq B(s)$ for all $s \in S$. 
In particular, $\mathrm{supp} \circ \pi$ is one such restriction of $B$. 
From our choice of actions to be removed (resulting in $B$), we have the following.
\begin{align}
    \aVal(s) = {} & \max_{a \in A} \min_{s' \in \supp{s}{a}} R(s,a) + \gamma \aVal(s') & \text{this is just \Cref{eqn:aval}}\\
    {} = {} & \min_{a \in B(s)} \min_{s' \in \supp{s}{a}} R(s,a) + \gamma \aVal(s') & \text{from our choice of }B\\
    {} = {} & \min_{a \in \mathrm{supp}(\pi(s))} \min_{s' \in \supp{s}{a}} R(s,a) + \gamma \aVal(s') & \text{since } \mathrm{supp} \circ \pi \text{ is a restriction of }B\label{eqn:last-restriction}.
\end{align}
It follows that the antagonistic values are also the (unique) solution to the system consisting of \Cref{eqn:last-restriction} for all $s \in S$. 
This last system is the same as the one prescribed by \Cref{eqn:aval-pol} to obtain the antagonistic value of $\pi$. 
We thus get that $\aVal(s) = \aVal^\pi(s)$ for all $s \in S$, which concludes the proof.\qed

\subsection{Completeness}
This part of the argument is the easiest. 
Note that the first step of the procedure removes state-action pairs $\tuple{s,a}$ that would decrease the antagonistic value of the $s$. 
By \Cref{lem:cali}, all worst-case optimal policies $\pi$ avoid such state-action pairs, lest some trajectory consistent with them would violate the guarantees of the lemma. 
This means the set of worst-case optimal policies $\underline{\pi}$ of the pMDP is preserved after removing those state-action pairs. 
Hence, the procedure does not miss any worst-case optimal policies that hit $\mathcal{I}$ almost surely.\qed

\section{Proof of \Cref{thm:qval-qval-pruning}}
Let $\pi$ be a policy, $\theta \in \Theta_{\mathrm{gp}}$, and $\tuple{s,a}$ a state-action pair. 
We prove that if $a \in \mathrm{supp}(\pi(s))$ and
\(
V_\theta(s) > Q_\theta(s,a),
\)
then $V_\theta(s) > V_\theta^\pi(s)$.

First, note that $V_\theta(s) \geq Q_\theta(s,b) \geq Q_\theta^\pi(s,b)$, for all $b \in A$, by definition of both values. 
Now, we have the following inequalities:
\begin{align*}
    V^\pi_\theta(s) = {} & \sum_{b \in A} \pi(b \given s)Q^\pi_\theta(s,b) & \text{by definition}\\
    {} \leq {} & \sum_{b \in A} \pi(b \given s)Q_\theta(s,b) & \text{since } Q_\theta(s,b) \geq Q_\theta^\pi(s,b)\\
    {} = {} & \sum_{b \in A \setminus a} \pi(b \given s)Q_\theta(s,b) + \pi(a \given s) Q_\theta(s,a)\\
    {} \leq {} & \sum_{b \in A \setminus a} \pi(b \given s)V_\theta(s) + \pi(a \given s) Q_\theta(s,a) & \text{since } V_\theta(s) \geq Q_\theta(s,b)\\
    {} = {} & (1 - \pi(a \given s))V_\theta(s) + \pi(a \given s) Q_\theta(s,a) & \text{because the sum is a convex combination}\\
    {} < {} & (1 - \pi(a \given s))V_\theta(s) + \pi(a \given s) V_\theta(s) & \text{by assumption}\\
    {} = {} & V_\theta & \text{since } 0 < \pi(a \given s) \leq 1.
\end{align*}
This concludes the proof. \qed

\section{Proof of \Cref{thm:aval-qval-pruning}}
Let $\pi$ be a policy, $\theta \in \Theta_{\mathrm{gp}}$, and $\tuple{s,a}$ a state-action pair. 
We prove that if $a \in \mathrm{supp}(\pi(s))$ and
\(
\mathbf{aVal}(s) > Q_\theta(s,a),
\)
then $V_\theta(s) > V_\theta^\pi(s)$. 
Moreover, if $\aVal^\pi(s) = \aVal(s)$ and $\Pr^\pi_\theta\left(\tau_{\mathcal{I}} < \infty \given s_0 = s\right) = 1$, for some $\pi$, the nonstrict version of the equation suffices.

From \Cref{lem:expected_value_bound_for_policy} and the condition of the claim, it follows that $V_\theta(s) > Q_\theta(s,a)$. 
Then, \Cref{thm:qval-qval-pruning} implies the result. 

For the second part of the claim, we make use of \Cref{lem:strict-nonstrict} to get $V_\theta(s) > \aVal(s) \geq Q_\theta(s,a)$ before again invoking \Cref{thm:qval-qval-pruning}. \qed

\section{Proof of \Cref{thm:aval-cval-pruning}}
Let $\pi$ be a policy and $\tuple{s,a}$ be a state-action pair. 
We prove that if  $a \in \mathrm{supp}(\pi(s))$ and
\[
\mathbf{aVal}(s) > \max_{s' \in \supp{s}{a}} R(s,a) + \gamma \mathbf{cVal}(s'),
\]
then for all $\theta \in \Theta_{\mathrm{gp}}$ we have $V_\theta(s) > V_\theta^\pi(s)$. 
Moreover, if $\aVal^\pi(s) = \aVal(s)$ and $\Pr^\pi_\theta\left(\tau_{\mathcal{I}} < \infty \given s_0 = s\right) = 1$, for some $\pi$, the nonstrict version of the equation suffices.

Assume that the condition of the claim is true. We get the following relations.
\begin{align*}
    Q_\theta(s,a) = {} & R(s,a) + \gamma \sum_{s' \in S} P_\theta(s' \given s, a) V_\theta(s') & \text{by definition}\\
    {} \leq {} & R(s,a) + \gamma \sum_{s' \in S} P_\theta(s' \given s,a) \cVal(s') & \text{by \Cref{lem:expected_value_bound_for_policy}}\\
    {} = {} & \sum_{s' \in S} P_\theta(s' \given s,a) \left( R(s,a) + \gamma \cVal(s') \right) & \text{because the sum is a convex combination}\\
    {} < {} & \sum_{s' \in S} P_\theta(s' \given s,a) \aVal(s) & \text{by assumption}\\
    {} = {} & \aVal(s) & \text{again, the sum is a convex combination}.\\
\end{align*}
We have just established that $\aVal(s) > Q_\theta(s,a)$ for the first part of the claim and that the nonstrict version of the inequality holds for the second part of the claim. 
The result thus follows from \Cref{thm:aval-qval-pruning}.\qed

\section{Overview of the benchmarks}
\label{app:benchmak_overview}

\paragraph{Behavior policies.}
For each environment, we construct a behavior policy to collect the datasets and serve as a baseline for (p)SPIBB to bootstrap.
This behavior policy is obtained by first computing an optimal (deterministic) policy using policy iteration. 
For each state, we make the policy stochastic by perturbating the probability mass by some parameter $\alpha \in [0,1]$.
Specifically, we remove probability mass $\alpha$ from the optimal action and re-distribute it uniformly over the other actions. 
Formally, the behavior policy is constructed as follows:
\[\pi_B(a \mid s) = \begin{cases}
    1-\alpha & \textbf{if } \pi(a \mid s) = 1,\\
    \frac{\alpha}{\lvert A \rvert} & \textbf{if }\pi(a \mid s) = 0.
\end{cases} 
\]
An overview of the dimensions and perturbation parameter $\alpha$ of all benchmarks can be found in \Cref{tab:overview_benchmarks}. 
All benchmarks, except Rock-Paper-Scissors, were derived from existing literature, as referenced in the main text.

\begin{table}[b]
    \centering
    \begin{tabular}{ll|cccc}
    \toprule
        Benchmark & Source & $\lvert S\rvert$ & $\lvert A\rvert$ & $\lvert X \rvert$ & $\alpha$\\
    \midrule
        Gridworld & \cite{DBLP:conf/icml/LarocheTC19} & 25 & 4 & 1 & $\nicefrac{1}{2}$ \\
        Resource gathering & \cite{DBLP:conf/icml/BarrettN08} & 376 & 4 & 1 & $\nicefrac{1}{5}$ \\
        Pac-Man & \cite{DBLP:conf/concur/0001KJSB20} & 498 & 5 & 0 & $\nicefrac{1}{20}$ \\
        Rock-paper-scissors & This paper & 1321 & 3 & 9 & $\nicefrac{1}{20}$\\
        Taxi & \cite{DBLP:journals/jair/Dietterich00} & 501 & 6 & 300 & $\nicefrac{1}{20}$ \\
    \bottomrule
    \end{tabular}
    \vspace{1em}
    \caption{Overview of the different benchmarks in terms of number of states, actions, parameters, and the $\alpha$ used to construct the behavior policy.}
    \label{tab:overview_benchmarks}
\end{table}

\subsection{Environments}

\paragraph{Rock-Paper-Scissors.}
Rock-Paper-Scissors (RPS) is a custom benchmark we introduce in this paper.
The RPS benchmark is constructed as a $20$-round game of rock-paper-scissors, in which the player receives a reward of $+1$ for winning, $-1$ for losing, and $0$ in case of a draw. 
The opponent of the player stochastically chooses whether to play rock, paper, or scissors, making use of a non-uniform distribution. 
This non-uniformity induces a bias that the player can exploit. 
The bias used in the RPS benchmarks is similar to an experiment in that paper and involves the opponent being more likely to choose a play that would have beaten the player in the previous round.
This experiment of finding an optimal strategy while playing against a biased opponent is inspired by a paper~\cite{BROCKBANK2024101654}, where the authors describe experiments where human participants had to play RPS against a biased robot. 

The parametric structure is implemented by labeling each transition in the pMDP with one of the following 9 parameters:
\[\begin{array}{lll}
p_{\mathrm{rock},    \mathrm{rock}}, &
p_{\mathrm{paper},   \mathrm{rock}}, &
p_{\mathrm{scissors},\mathrm{rock}}, \\
p_{\mathrm{rock},    \mathrm{paper}}, &
p_{\mathrm{paper},   \mathrm{paper}}, &
p_{\mathrm{scissors},\mathrm{paper}}, \\
p_{\mathrm{rock},    \mathrm{scissors}}, &
p_{\mathrm{paper},   \mathrm{scissors}}, &
p_{\mathrm{scissors},\mathrm{scissors}},
\end{array}\]
where $p_{a,b}$ encodes the probability that the opponent chooses action $a$ if the player played $b$ in the previous round. In \Cref{fig:example_rps_sa_pair}, we can see the parametric distribution of one state-action pair. The chosen action is ``paper'' and in the previous round the player chose ``rock''. The bias is implemented by setting $p_{\mathrm{paper},   \mathrm{rock}} = 0.4$ and $p_{\mathrm{rock}, \mathrm{rock}} = p_{\mathrm{scissors}, \mathrm{rock}} = 0.3$
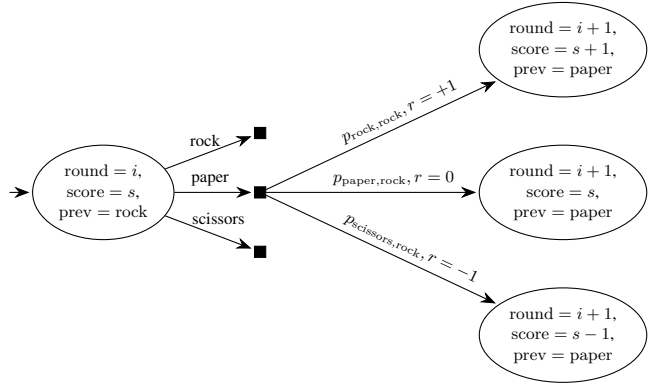
\begin{figure}[h]
\centering
    \tikzset{elliptic state/.style={draw,ellipse}}
\begin{tikzpicture}[shorten >=1pt,auto,node distance=.9 cm, scale = 0.7, transform shape, ->]
        \clip (-2.1,-4.2) rectangle (11,4); 
        \tikzstyle{action} = [fill=black, shape=rectangle, draw]

        \node[elliptic state, align=center,  initial=left, initial text={}](s0){$\mathrm{round}=i$,\\ $\mathrm{score}=s$,\\ $\mathrm{prev}=\mathrm{rock}$};

        \node[action](s0_b)[right= 1.5cm of s0]{};
        \node[action](s0_a)[above= of s0_b]{};
        \node[action](s0_c)[below= of s0_b]{};

        \node[elliptic state, align=center, right= 4cm of s0_b](s4){$\mathrm{round}=i+1$,\\ $\mathrm{score}=s$, \\ $\mathrm{prev}=\mathrm{paper}$};
        \node[elliptic state, align=center, above = of s4 ](s3){$\mathrm{round}=i+1$,\\ $\mathrm{score}=s+1$, \\ $ \mathrm{prev}=\mathrm{paper}$};
        \node[elliptic state, align=center, below = of s4 ](s5){$\mathrm{round}=i+1$,\\ $\mathrm{score}=s-1$, \\ $\mathrm{prev}=\mathrm{paper}$};

        \draw (s0) edge [above, pos=0.45] node [label={[xshift=-0.0cm, yshift=-0.2cm]paper}]{} (s0_b);
        \draw (s0) edge [above, pos=0.45] node [label={[xshift=-0.0cm, yshift=-0.2cm]rock}] {} (s0_a);
        \draw (s0) edge [above, pos=0.45] node [label={[xshift=0.25cm, yshift=-0.2cm]scissors}] {} (s0_c);

        \draw[auto]
        (s0_b) edge [above, pos=0.60, sloped] node {$p_{\mathrm{rock},\mathrm{rock}}, r=+1$}  (s3)
        (s0_b) edge [swap, above, pos=0.60] node {$p_{\mathrm{paper}, \mathrm{rock}}, r=0$} (s4)
        (s0_b) edge [swap, above, pos=0.60, sloped] node {$p_{\mathrm{scissors}, \mathrm{rock}}, r=-1$} (s5);
    \end{tikzpicture}
    \vspace{1em}
    \caption{Example of a biased distribution for a state-action pair in the Rock-Paper-Scissors pMDP.}
    \label{fig:example_rps_sa_pair}
    \vspace{3em}
\end{figure}

\paragraph{Gridworld.} The original Gridworld benchmark \cite{DBLP:conf/icml/LarocheTC19} consists of a 5-by-5 grid that the player has to navigate. The rewards for all transitions are zero, except the reward for reaching a final state. Additionally, there are some obstacles. The actions are left, right, up, and down. Once an action is taken, the player stochastically ends up in the chosen direction or in one of the three other directions. In the original MDP, the probabilities over successor states are fixed. In our pMDP, however, the player goes in the chosen direction with probability $p$, and accidentally ends up in one of the other three possible states each with probability $\frac{1-p}{3}$. This probability $p$ is a parameter of the pMDP and is shared by all state-action pairs.

\paragraph{Taxi.} This is a modification of the Taxi benchmark \cite{DBLP:journals/jair/Dietterich00}. In the original model that passenger is located at a random location, and the taxi starts at a random location. Each passenger-taxi-position combination is one initial state. The original model has a uniform distribution over initial states. In our pMDP version we add a fresh unique initial state that has a transition to each of the original initial states. Each such a transition is labeled with a fresh unique parameter.

\paragraph{Resource gathering.} The Resource Gathering benchmark consists of a 5-by-5 grid across which the player can move. On this grid there are gems and gold that the player has to pick up and bring back to the home position. However, there are also enemies located on certain grid squares. If the player steps on such an enemy square, the player has probability $p$ to be attacked and $1-p$ to be left alone. The state space keeps track of where the player is, whether they have picked up a gem or gold, and how much gems and gold have already been collected. The actions are left, right, up, and down. In the original benchmark MDP \cite{DBLP:conf/icml/BarrettN08}, the attack probability was fixed at $0.1$, while in our version it has been abstracted as a parameter in the pMDP. This parameter $p$ is shared by all actions leading into states that represent enemy-inhabited locations.

\paragraph{Pac-Man.} The Pac-Man benchmark from \cite{DBLP:conf/concur/0001KJSB20} was mostly unchanged, except that we replaced the floating-point probabilities with rational numbers, such that we use an SMT-solver to reason about these probabilities. This benchmark does not have any parameters, and is only used to demonstrate the game-based pruning technique.

\end{document}